%% file: TRL_revised_v3.tex
\documentclass[lettersize,journal]{IEEEtran}

\IEEEoverridecommandlockouts                  

\usepackage{mathtools}
\usepackage{gensymb}
\usepackage[export]{adjustbox}
\usepackage{color}
\usepackage{graphicx}
\usepackage{booktabs}
\usepackage{epsfig} 
\usepackage{times} 
\usepackage{amsmath} 
\usepackage{amssymb}
\usepackage{amsthm}
\usepackage{leftidx}
\usepackage{stfloats}
\usepackage[noadjust]{cite}
\newtheorem{assumption}{Assumption}

\newtheorem{proposition}{Proposition}

\newtheorem{theorem}{Theorem}[section]

\newtheorem{lemma}[theorem]{Lemma}

\usepackage{siunitx}
\usepackage{nomencl}
\makenomenclature
\usepackage{hyperref}
\usepackage{csquotes}

\usepackage{enumitem}
\usepackage{tablefootnote}
\usepackage{ushort}

\usepackage{algorithm}
\usepackage{algorithmic}

\usepackage{stackengine}
\usepackage{multirow}  

\usepackage[caption=false,font=footnotesize]{subfig}

\title{\fontsize{22}{24}\selectfont
Learning Cross-Coupled and Regime-Dependent Dynamics for Aerial Manipulation}

\author{
Rishabh Dev Yadav$^{1}$, Samaksh Ujjawal$^{2}$,
Shivansh Pratap Singh$^{2}$, Akshat Singh$^{3}$,\\
Sihao Sun$^{4}$, Spandan Roy$^{2}$, and Wei Pan$^{5}$
\thanks{
$^{1}$Department of Computer Science, The University of Manchester, UK
(e-mail: rishabh.yadav@postgrad.manchester.ac.uk).
$^{2}$IIIT Hyderabad, India
(e-mails: \{samaksh.ujjawal, shivansh.singh\}@research.iiit.ac.in;
spandan.roy@iiit.ac.in).
$^{3}$Institute for Systems and Robotics, Instituto Superior Técnico,
University of Lisbon, Portugal (e-mail: ist430655@tecnico.ulisboa.pt).
$^{4}$Department of Cognitive Robotics, Delft University of Technology,
The Netherlands (e-mail: s.sun-2@tudelft.nl).
$^{5}$School of Engineering, Newcastle University, UK
(e-mail: wei.pan2@newcastle.ac.uk).
Corresponding author: Rishabh Dev Yadav.
}
}

\begin{document}
\include{our_settings}

\maketitle
\thispagestyle{empty}
\pagestyle{empty}
\setlength{\belowcaptionskip}{-10pt}

\begin{abstract}
Accurate dynamics modeling of aerial manipulators is challenging because
platform--manipulator coupling, history-dependent interactions, and changing
payload or manipulator configurations produce cross-coupled, temporal, and
nonstationary residual dynamics. We propose a structured encoder--decoder
framework in which a nonlinear history encoder captures temporal and
cross-variable dependencies, while a lightweight linear latent decoder
supports efficient deployment-time adaptation. Its linear-in-parameter
structure enables closed-form Bayesian updates with consistency-driven
covariance inflation and integration with real-time model predictive control
(MPC). Experiments on a real aerial-manipulation platform demonstrate improved
residual-prediction accuracy, rapid adaptation following dynamics changes,
and lower MPC tracking error under changing payload conditions. The results
show that combining structured history-dependent representation learning with
lightweight online adaptation provides an effective predictive model for
aerial manipulation under coupled and nonstationary dynamics.
\end{abstract}


\vspace{-8pt}

\section{Introduction}

Aerial manipulators combine multirotor platforms with articulated arms for
tasks such as inspection, contact interaction, grasping, and payload
transport~\cite{orsag2018aerial,yadav2025integrated, singh2026aerograb}. Their performance
depends critically on predictive dynamics models for trajectory optimization
and model-based control.
Obtaining such models is difficult because manipulator motion continuously
changes the inertia distribution and center of mass, producing
configuration-dependent interaction between the platform and arm
~\cite{ollero2022past, yadav2024modular}, while aerodynamic
interactions introduce additional effects that are difficult to model
analytically~\cite{bauersfeld2021neurobem}. Moreover, payload 
pickup or release, contact, grasping,
object placement~\cite{mishra2026aeroplace}, and
manipulator reconfiguration can alter these dynamics
during operation. The resulting modeling
error is therefore not merely uncertain in magnitude: it can be
\emph{cross-coupled}, depend on recent system evolution, and change with the
operating regime.

This combination is particularly challenging for predictive control. A model
that treats platform and manipulator variables independently can miss their
interaction, an instantaneous predictor can discard information contained in
recent motion, and a fixed offline model can become inaccurate after the
operating regime changes. Conversely, adapting a large nonlinear dynamics
model during closed-loop operation can be computationally expensive and
data-demanding. The central challenge is therefore to retain an expressive
representation of coupled, history-dependent dynamics while concentrating
deployment-time adaptation in a sufficiently small model component.

\begin{figure*}[t]
\centering
\vspace{-2mm}
\includegraphics[width=0.95\linewidth]{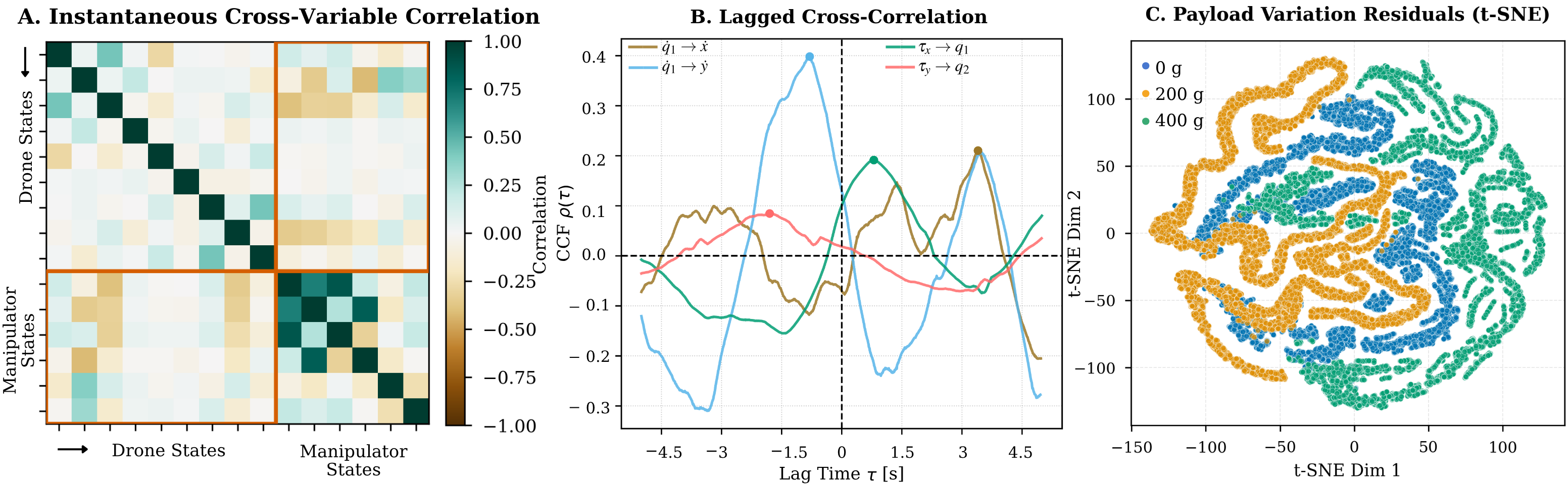}
\caption{
{\color{black}
Diagnostic analysis motivating the residual-model structure.
(A) Zero-lag correlation matrix of synchronized platform and manipulator
variables.
(B) Normalized cross-correlation of joint velocity ($\dot{q}_1$) on aerial base translational velocities ($\dot{x}, \dot{y}$) and base Control inputs ($\tau_x, \tau_y$) on arm joint positions ($q_1, q_2$) 
from $-500$ to $500$ samples ($-5$ to $5$\,s at $100$\,Hz); nonzero-lag
extrema indicate temporal statistical dependence beyond the current sample.
(C) Joint t-SNE embedding of standardized one-step residual vectors
$\delta_{k+1}=x_{k+1}-f_{\mathrm{nom}}(x_k,u_k)$ for the
$0$, $200$, and $400$\,g payload regimes used during offline training.
}
}
\label{fig:motivation_fig}
\end{figure*}

{\color{black}
Rather than motivating these requirements only conceptually,
Fig.~\ref{fig:motivation_fig} examines them directly in measured
aerial-manipulator data. 
For synchronized scalar signals $s_i[n]$ and
$s_j[n]$, we compute the normalized cross-correlation at lag $\kappa$,
\begin{equation*} \small
r_{ij}[\kappa]=
\frac{\sum_{n\in\mathcal I_\kappa}\tilde s_i[n]\tilde s_j[n+\kappa]}
{\sqrt{\sum_{n\in\mathcal I_\kappa}\tilde s_i[n]^2}
 \sqrt{\sum_{n\in\mathcal I_\kappa}\tilde s_j[n+\kappa]^2}},
\end{equation*}
where $\tilde s_i,\tilde s_j$ are centered signals and
$\mathcal I_\kappa$ contains the valid overlapping samples.


At $\kappa=0$, Fig.~\ref{fig:motivation_fig}(A) exhibits substantial
off-diagonal dependence between platform and manipulator variables. This suggests that independent per-channel processing may discard useful
cross-variable structure and motivates a representation that processes the
multivariate platform, manipulator, and control signals jointly.

Figure~\ref{fig:motivation_fig}(B) evaluates the same statistic over
$\kappa\in[-500,500]$ samples, corresponding to
$\tau_\kappa=\kappa/f_s\in[-5,5]$\,s at $f_s=100$\,Hz. Several relationships
remain pronounced away from zero lag, indicating that information relevant
to the current dynamics is distributed across recent measurements rather
than confined to the instantaneous state--input pair. This motivates a
finite-history representation. We emphasize that these nonzero-lag extrema
are used only as evidence of lagged statistical dependence and are not
interpreted as causal delays or identification of a particular physical
interaction mechanism.

The remaining difficulty is that even an expressive history-dependent model
may cease to be accurate when the operating regime changes.
Figure~\ref{fig:motivation_fig}(C) examines this effect using the same
one-step residual subsequently modeled in this work,
$\delta_{k+1}=x_{k+1}-f_{\mathrm{nom}}(x_k,u_k)$.
Residuals from the $0$, $200$, and $400$\,g training regimes are
standardized and jointly embedded using t-SNE (perplexity $30$,
$1000$ iterations), with payload labels used only for coloring.
The payload-dependent neighborhood structure is qualitatively consistent
with a change in residual distribution across operating regimes. Because
t-SNE primarily preserves local neighborhood structure, this visualization
is not used as quantitative evidence of separability or as an input to the
learning pipeline.

These observations lead directly to three design requirements:
\emph{joint cross-variable representation}, \emph{finite-history
conditioning}, and \emph{deployment-time adaptation}.
Residual learning is well suited to this decomposition because the nominal
model retains the dominant rigid-body structure, allowing the learned
component to focus on unmodeled interaction and regime-dependent effects
~\cite{saviolo2023learning}. The remaining design tension is
between expressiveness and adaptability: the residual representation should
be nonlinear enough to capture coupled temporal structure, yet its
deployment-time update should remain lightweight enough for closed-loop
operation. This motivates our structured encoder--decoder formulation, in
which the nonlinear history representation is learned offline and retained
during deployment, while only a low-dimensional linear latent decoder is
adapted online.
}

\subsection{Related Work}

{\color{black}
Adaptive and observer-based aerial-manipulator controllers
compensate uncertain coupling, contact effects, and
center-of-mass variations within the control
law~\cite{yadav2024modular,yadav2025integrated,
liang2025observer}; switched adaptive control additionally
addresses abrupt dynamics transitions~\cite{yadav2026switched}. Reinforcement learning enables direct policy learning
for aerial control and manipulation~\cite{deshmukh2026global},
while vision-language models support high-level
skill selection~\cite{mishra2026aermanivlm}.
In contrast, our approach learns an explicit, online-adaptive
predictive dynamics model for constrained MPC.
Learning-based aerial dynamics models employ DNNs
~\cite{shi2019neural,salzmann2023real}, Gaussian processes
~\cite{cao2024computation}, temporal networks
~\cite{saviolo2022physics}, and neural stochastic models
~\cite{djeumou2023learn}. Although expressive, fixed offline
models can suffer from distribution shifts and compounding
prediction errors~\cite{lambert2022investigating}.

Online dynamics adaptation has been explored through network
or last-layer updates~\cite{saviolo2023active,zhou2025simultaneous,
wei2025mlmpcc,jiahao2023online}, learned representations and
prototypes~\cite{o2022neural,gu2024proto}, and GP, Bayesian,
or adaptive-control mechanisms~\cite{gahlawat2020l1,richards2021adaptive}.
However, nonlinear model updates can be computationally
expensive, while predefined model combinations may limit
flexibility under unseen regimes. Sequence models
~\cite{vaswani2017attention,rao2024learning,das2025dronediffusion}
capture temporal dependencies, whereas multivariate forecasting
architectures~\cite{zhang2023crossformer}
model cross-variable interactions.
Recent structured residual learning employs RLS-based
latent adaptation~\cite{ujjawal2026learn}.
Our framework instead uses innovation-consistency-driven
covariance reopening for recursive decoder adaptation
and integrates the learned residual with MPC.
}

\subsection{Contributions}

We propose a structured encoder--decoder framework for adaptive
aerial-manipulator residual dynamics learning with two main contributions:
{\color{black}
\begin{enumerate}[leftmargin=*,itemsep=0pt,topsep=2pt]
\item \textbf{Structured residual representation learning:}
a nonlinear history encoder jointly models temporal and cross-variable
dependencies, yielding a compact shared feature basis for coupled platform
and manipulator residuals.

\item \textbf{Consistency-driven online decoder adaptation:}
a linear latent decoder is updated recursively using Bayesian estimation,
while innovation consistency reopens its covariance under model mismatch,
enabling deployment-time adaptation without retraining the encoder.
\end{enumerate}
}
The framework is evaluated through residual prediction, online adaptation,
and closed-loop MPC under changing payloads and multiple trajectory families.

\section{Methodology}
\label{sec:method}

\subsection{Problem Formulation}
\label{sec:problem_formulation}

We consider $x_{k+1}=f(x_k,u_k)$, with
$x_k\in\mathbb{R}^{d}$ and $u_k\in\mathbb{R}^{m}$.
A physics-based model $f_{\mathrm{nom}}$ captures the dominant rigid-body
dynamics, while the remaining error is represented by a history-dependent
additive residual:
\begin{equation} 
x_{k+1}
=
f_{\mathrm{nom}}(x_k,u_k)
+
f_{\Delta}(H_k)
+
\varepsilon_k
\label{eq:true_system}
\end{equation}
where
$\varepsilon_k\sim
\mathcal{N}(0,\mathrm{diag}(\sigma_1^2,\ldots,\sigma_d^2))$. Let $\chi_k=[x_k^\top,u_k^\top]^\top\in\mathbb{R}^{n_c}$,
$n_c=d+m$. The finite history is
$H_k=[\chi_{k-h},\ldots,\chi_k]^\top
\in\mathbb{R}^{(h+1)\times n_c}$, where $h$ is the history length.

The one-step residual target is
$\delta_{k+1}:=x_{k+1}-f_{\mathrm{nom}}(x_k,u_k)
=f_{\Delta}(H_k)+\varepsilon_k$.
For quaternion components, the correction is represented in ambient state
coordinates and the quaternion is renormalized after propagation, as detailed
in Appendix~\ref{app:mpc_impl}.

To retain nonlinear representation capacity while enabling lightweight online
adaptation, we parameterize
\begin{equation}
f_{\Delta}(H_k)\approx\Theta z_k,
\qquad
z_k=\phi(H_k)\in\mathbb{R}^{\ell},
\qquad
\Theta\in\mathbb{R}^{d\times\ell}.
\label{eq:latent_residual_model}
\end{equation}
Thus, $H_k\mapsto f_{\Delta}(H_k)$ remains nonlinear through the encoder
$\phi$, while the decoder is linear in the parameters adapted online, following
the learned-basis-function principle~\cite{o2022neural}.

{\color{black}

}

{\color{black}
The following conditions are used only for the local estimator analysis in
Appendix~\ref{app:theory}, not for global closed-loop stability.

\begin{assumption}[Bounded latent representation]
\label{ass:bounded_features}
Over the considered operating region, there exists $c_z>0$ such that
$\|z_k\|_2\leq c_z$.
\end{assumption}

\begin{assumption}[Local latent-model approximation]
\label{ass:realizability}
Over a given operating interval, each residual output admits
\begin{equation}
\delta_{k+1}^{(j)}
=
z_k^\top\theta_j^\star+\xi_{k,j}+\varepsilon_k^{(j)},
\qquad
\theta_j^\star\in\mathbb{R}^{\ell},
\quad
|\xi_{k,j}|\leq\bar{\xi}_j .
\label{eq:local_latent_approx}
\end{equation}
\end{assumption}

These are local regularity conditions: bounded features and a
linear-in-parameter representation with bounded approximation mismatch; they
do not require a single decoder to represent all operating regimes. Similar
conditions are standard in robust adaptive estimation~\cite{ioannou1996robust},
and online-adapted linear coefficients over learned features have been used in
learning-based flight control~\cite{o2022neural}.
}

\begin{figure*}[t]
\centering
\includegraphics[width=\linewidth, height=4.2in]{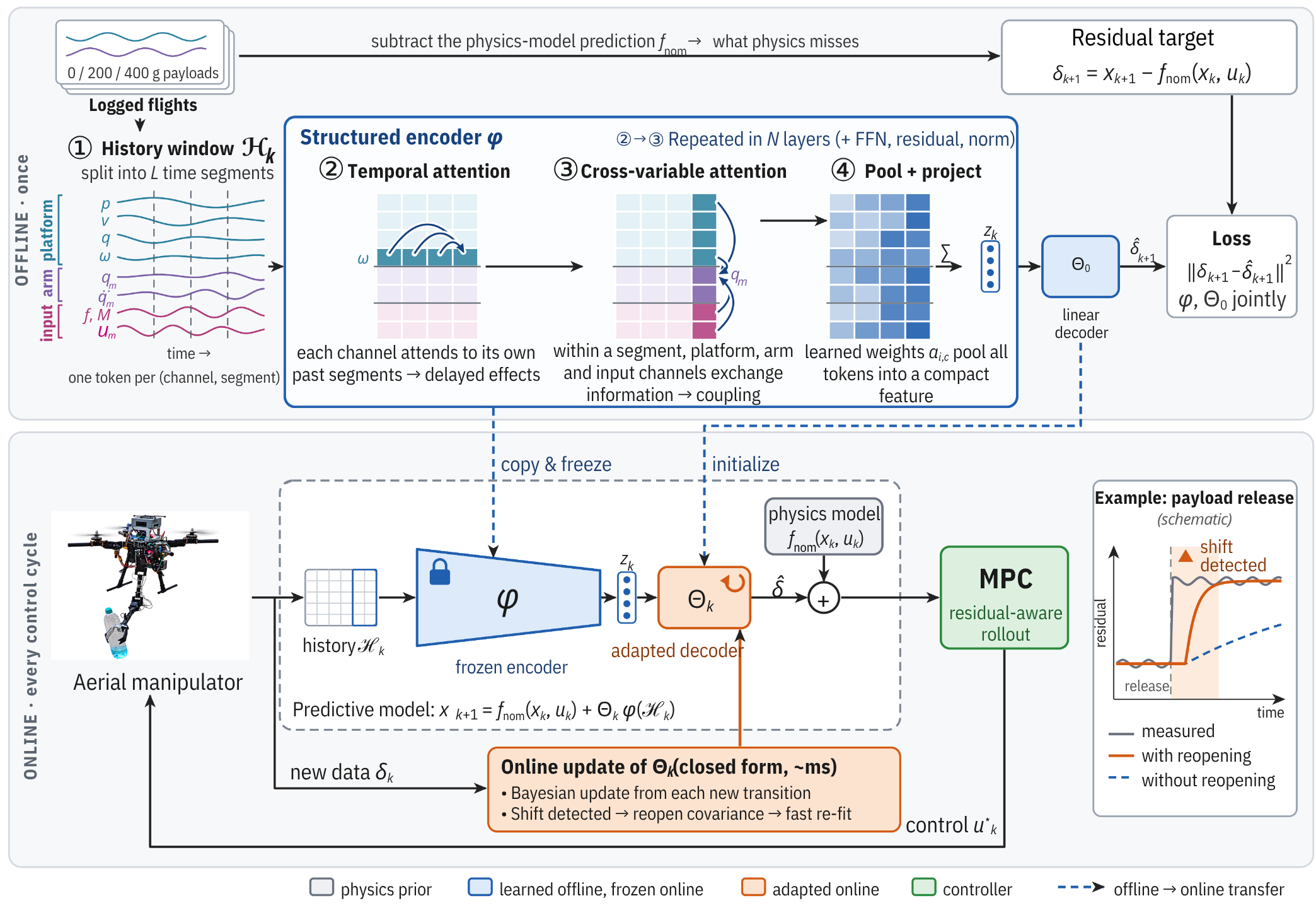}
\vspace{-0.5cm}
\caption{Overview of the proposed framework. Offline, a structured nonlinear
history encoder and linear latent decoder are jointly trained for residual
prediction. During deployment, the encoder is frozen while the decoder is
updated recursively using completed transitions and innovation consistency.
The adapted residual model is then combined with the nominal dynamics for
closed-loop MPC.
}
\label{fig:online_update}
\end{figure*}

\subsection{Structured History Encoder}
\label{sec:encoder}

The encoder $\phi(H_k)$ extracts temporal and cross-variable structure,
motivated by Fig.~\ref{fig:motivation_fig}(A)--(B). Each channel history is
divided into segments $s_{i,c}\in\mathbb{R}^{L_{\mathrm{seg}}}$ of length
$L_{\mathrm{seg}}$, with end-padding when required. Each segment is embedded
as $h_{i,c}=W_{\mathrm{emb}}s_{i,c}+e_{i,c}$, where
$W_{\mathrm{emb}}\in
\mathbb{R}^{d_{\mathrm{model}}\times L_{\mathrm{seg}}}$ and $e_{i,c}$
encodes temporal position and channel identity. Stacking the tokens gives
$Z\in\mathbb{R}^{L\times n_c\times d_{\mathrm{model}}}$, where $L$ is the
number of temporal segments.

Each encoder layer applies temporal self-attention independently across
segments, $Z^{(t)}=\mathcal{A}_t(Z)$, followed by cross-variable attention
within each segment, $Z^{(c)}=\mathcal{A}_c(Z^{(t)})$. This ordering is
deliberate: temporal attention first extracts the recent evolution of each
signal, after which cross-variable attention exchanges information among the
platform, manipulator, and control channels at the resulting temporal
contexts. Thus, the representation captures both history dependence and
cross-subsystem coupling rather than processing each channel independently.
Both operations use multi-head self-attention, feedforward transformations,
residual connections, and normalization; multiple layers are stacked.


The final features are attention-pooled and projected as
\begin{equation}
\tilde z_k=\sum_{i,c}a_{i,c}Z^{(c)}_{i,c},
\qquad
z_k=f_{\mathrm{proj}}(\tilde z_k)\in\mathbb{R}^{\ell},
\label{eq:encoder_output}
\end{equation}
where $a_{i,c}$ are normalized learned pooling weights and
$f_{\mathrm{proj}}:\mathbb{R}^{d_{\mathrm{model}}}\rightarrow
\mathbb{R}^{\ell}$. The resulting $z_k$ provides the shared nonlinear feature
basis in \eqref{eq:latent_residual_model}.

\subsection{Offline Representation Learning}
\label{sec:offline_training}

Writing the decoder row-wise as
$\Theta=[\theta_1^\top,\ldots,\theta_d^\top]^\top$,
$\theta_j\in\mathbb{R}^{\ell}$, gives
$\delta_{k+1}^{(j)}=z_k^\top\theta_j+\varepsilon_k^{(j)}$.
The encoder and initial decoder are jointly trained on
$\mathcal{D}_{\mathrm{train}}
=\{(H_k,\delta_{k+1})\}_{k=1}^{N_{\mathrm{tr}}}$ by minimizing
\begin{equation}
\min_{\phi,\Theta}
\sum_{k=1}^{N_{\mathrm{tr}}}
\left\|
\delta_{k+1}-\Theta\phi(H_k)
\right\|_2^2 .
\label{eq:offline_objective}
\end{equation}
This is the standard Gaussian least-squares objective used for offline
residual fitting.
After training, $\phi$ is frozen and the learned decoder $\Theta_0$
initializes the online adaptation in
Section~\ref{sec:online_adaptation}. Figure~\ref{fig:online_update}
summarizes the complete pipeline.

\section{Online Decoder Adaptation}
\label{sec:online_adaptation}

During deployment, the encoder $\phi$ is frozen and only the latent decoder
is adapted, allowing the residual mapping to vary with the operating regime
without retraining the nonlinear representation.

\subsection{Recursive Bayesian Decoder Update}
\label{sec:bayesian_update}

For residual output $j$,
$\delta_k^{(j)}=z_{k-1}^{\top}\theta_{k,j}+\varepsilon_{k-1}^{(j)}$,
where $z_{k-1}=\phi(H_{k-1})$,
$\theta_{k,j}\in\mathbb{R}^{\ell}$, and
$\varepsilon_{k-1}^{(j)}\sim\mathcal{N}(0,\sigma_j^2)$.
{\color{black}
Cross-variable coupling is represented in the conditional mean through the
shared nonlinear feature $z_k=\phi(H_k)$, which jointly encodes platform,
manipulator, and control histories. The output-wise covariance model used
below therefore simplifies the online uncertainty update without imposing
independence on the predicted residual mean.
}

To accommodate time-varying dynamics, the decoder parameters follow
\begin{equation}
\theta_{k,j}
=
\theta_{k-1,j}+w_{k,j},
\qquad
w_{k,j}\sim\mathcal{N}(0,Q_{k,j}),
\label{eq:parameter_random_walk}
\end{equation}
with posterior
$\theta_{k,j}\mid\mathcal{D}_k
\sim\mathcal{N}(\mu_{k|k,j},P_{k|k,j})$.
The initial covariance and observation variances are calibrated from offline
training data (Appendix~\ref{app:decoder_online}), yielding
$P_{0|0,j}\succ0$ and $\sigma_j^2>0$
~\cite{simon2006optimal,ljung1999system}.


Given the previous posterior, the parameter prior is
$\mu_{k|k-1,j}=\mu_{k-1|k-1,j}$ and
$P_{k|k-1,j}=P_{k-1|k-1,j}+Q_{k,j}$.
The corresponding residual prediction and variance are
$\hat{\delta}_{k|k-1}^{(j)}
=z_{k-1}^{\top}\mu_{k|k-1,j}$ and
$S_{k,j}=z_{k-1}^{\top}P_{k|k-1,j}z_{k-1}+\sigma_j^2$.

After observing $x_k$, the completed-transition residual becomes available
and the posterior is updated using $(z_{k-1},\delta_k^{(j)})$:
{\small
\begin{align}
K_{k,j}
&=
P_{k|k-1,j}z_{k-1}
\left(
z_{k-1}^{\top}P_{k|k-1,j}z_{k-1}+\sigma_j^2
\right)^{-1},
\nonumber\\
\mu_{k|k,j}
&=
\mu_{k|k-1,j}
+
K_{k,j}
\left(
\delta_k^{(j)}
-
z_{k-1}^{\top}\mu_{k|k-1,j}
\right),
\label{eq:recursive update}\\
P_{k|k,j}
&=
(I-K_{k,j}z_{k-1}^{\top})
P_{k|k-1,j}
(I-K_{k,j}z_{k-1}^{\top})^{\top}
+
\sigma_j^2K_{k,j}K_{k,j}^{\top}.
\nonumber
\end{align}
}

For robustness to extreme outliers, implementation replaces $\sigma_j^2$
in \eqref{eq:recursive update} by the Huber-weighted effective variance
$\sigma_{k,j,\mathrm{eff}}^2\geq\sigma_j^2$, defined in
Appendix~\ref{app:decoder_online}.

The adapted decoder is
$\Theta_k=[\mu_{k|k,1}^{\top},\ldots,\mu_{k|k,d}^{\top}]^{\top}
\in\mathbb{R}^{d\times\ell}$.
After applying $u_k$, the new history gives $z_k=\phi(H_k)$ and
$\hat{\delta}_{k+1|k}=\Theta_kz_k$, stored for the next completed
transition. The update costs $\mathcal{O}(d\ell^2)$ per cycle. Under consistent
stationary data, repeated observations reduce parameter uncertainty and hence
the responsiveness of subsequent Bayesian updates. This is desirable within a
fixed regime, but after an abrupt dynamics change the contracted covariance can
make adaptation temporarily conservative. Sections~\ref{sec:shift_detection}
and~\ref{sec:shift_adaptation} therefore monitor prediction consistency and
selectively reopen the covariance when the current decoder becomes
inconsistent with new observations. Estimator properties are summarized in
Appendix~\ref{app:theory}.

\subsection{Shift Detection via Innovation Consistency}
\label{sec:shift_detection}

After covariance contraction, abrupt dynamics changes can make the decoder
temporarily inconsistent with new observations. We therefore monitor
normalized innovations using a non-inflated reference decoder, initialized
identically and updated from the same transitions but without covariance
reopening. This prevents detector-triggered adaptation from altering the
uncertainty used to normalize subsequent innovations. 

For output $j$, the reference prediction is
$\hat{\delta}_{k|k-1}^{\mathrm{ref},(j)}
=z_{k-1}^{\top}\mu_{k|k-1,j}^{\mathrm{ref}}$, giving innovation
\begin{equation}
e_{k,j}
=
\delta_k^{(j)}
-
\hat{\delta}_{k|k-1}^{\mathrm{ref},(j)},
\label{err_prediction}
\end{equation}
with
$S_{k,j}^{\mathrm{ref}}
=z_{k-1}^{\top}P_{k|k-1,j}^{\mathrm{ref}}z_{k-1}+\sigma_j^2$
and normalized score
$\psi_{k,j}=e_{k,j}^2/S_{k,j}^{\mathrm{ref}}$.
The clipped EWMA statistic is
\begin{equation}
g_{k,j}
=
\alpha\min\!\left(\psi_{k,j},c_{\mathrm{clip}}^2\right)
+
(1-\alpha)g_{k-1,j},
\label{eq:ewma_update}
\end{equation}
where $\alpha\in(0,1]$ and $c_{\mathrm{clip}}$ is the innovation clipping
level.

Each channel is compared with a nominally calibrated threshold $\eta_j$.
Let $c_k=\sum_{j=1}^{d}\mathbf{1}[g_{k,j}>\eta_j]$.
If $c_k<n_{\min}$, set $g_k=c_k/n_{\min}$; otherwise,
\[
g_k
=
1+\frac{1}{n_{\min}}
\sum_{j\in\mathcal I_k}
\left(\frac{g_{k,j}}{\eta_j}-1\right),
\]
where $\mathcal I_k$ contains the $n_{\min}$ largest normalized scores
$g_{k,j}/\eta_j$. Thus, $g_k>1$ only when at least $n_{\min}$ channels
coherently exceed their thresholds. {\color{black}
Requiring coherent mismatch across multiple residual channels
reduces sensitivity to isolated measurement errors while
remaining responsive to broader changes in coupled
aerial-manipulator dynamics.
}
The monitor is empirically calibrated
rather than a formal hypothesis test; its scope is discussed in
Appendix~\ref{app:theory}.

\subsection{Adaptive Covariance Inflation}
\label{sec:shift_adaptation}

A detected mismatch reopens the decoder covariance relative to its initial
value,
\begin{equation}
P_{k|k-1,j}
=
P_{k-1|k-1,j}
+
\rho_kP_{0,j},
\qquad
P_{0,j}:=P_{0|0,j},
\label{eq:cov_reopening}
\end{equation}
equivalent to $Q_{k,j}=\rho_kP_{0,j}$ in
\eqref{eq:parameter_random_walk}. Using the previous consistency statistic,
\begin{equation} \small
\lambda_k
=
\max\!\left[
\lambda_{\min},
\frac{1}{
1+\beta\min([g_{k-1}-1]_+,c_g)}
\right],
~\rho_k
=
\gamma\left(\lambda_k^{-1}-1\right),
\label{eq:lambda_adaptive}
\end{equation}
where $[a]_+=\max(a,0)$, $\beta$ controls mismatch sensitivity,
$\lambda_{\min}$ bounds adaptation strength, $\gamma$ scales reopening, and
$c_g$ caps large consistency excursions.


For $g_{k-1}\leq1$, $\lambda_k=1$ and $\rho_k=0$; when $g_{k-1}>1$,
$\rho_k>0$ increases the prior covariance and hence adaptation capacity.
The reopening magnitude and covariance spectrum are bounded for numerical
robustness; Appendix~\ref{app:theory} summarizes the corresponding covariance
property.
Figure~\ref{fig:decoder_mechanism} summarizes the online mechanism: the
encoder remains fixed, while innovation consistency regulates decoder
covariance and hence its adaptation rate.

\begin{figure}[t]
\centering
\includegraphics[width=\linewidth, height=4.5in]{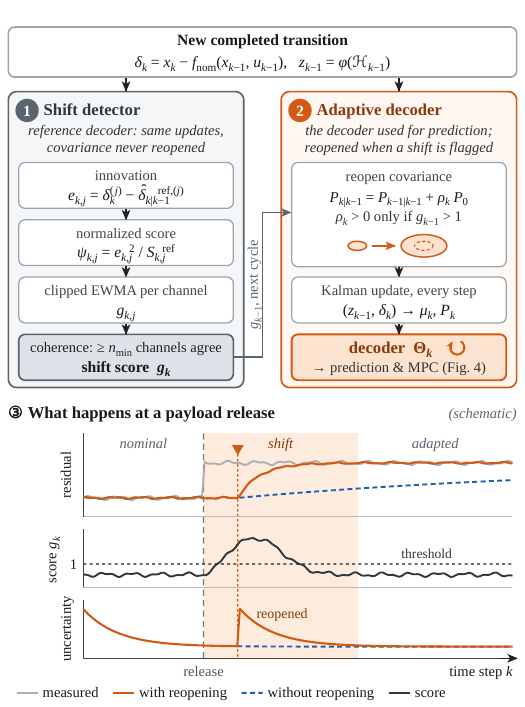}
\vspace{-0.4cm}
\caption{Illustration of consistency-driven online decoder adaptation.
The frozen encoder provides $z_k$, the linear decoder predicts the residual,
and coherent innovation mismatch triggers covariance reopening for rapid
Bayesian adaptation without resetting the decoder mean. 
}
\label{fig:decoder_mechanism}
\end{figure}

\section{Model Predictive Control for Tracking}
\label{sec:mpc}

The adapted residual model is incorporated into nonlinear receding-horizon
MPC. At each control cycle, the decoder is updated from the latest completed
transition and its posterior mean $\Theta_k$ is used for prediction. Figure~\ref{fig:mpc_pipeline} illustrates the residual-aware MPC
integration, while Algorithm~\ref{alg:method} summarizes the complete online
adaptation and control sequence.

{\color{black}
Directly recomputing the history encoder inside every optimizer iteration is
computationally expensive and recursively couples residual predictions to
optimizer-generated states. We instead shift the previous optimal sequence
$U_{k-1}^{\star}$ to generate a local reference rollout and evaluate the
history-dependent residual only along this rollout. For the first $K$ nodes,
\begin{equation}
\bar r_{k,i}=\Theta_k\phi(\bar H_{k,i}),
\qquad 0\leq i<K,
\label{eq:residual_rollout}
\end{equation}
where $\bar H_{k,i}$ is the corresponding state--input history. Beyond this
depth, $\bar r_{k,i}=\zeta_i\bar r_{k,K-1}$, where
\begin{equation}
\begin{aligned}
\zeta_i &=
\frac{1}{2}\!\left[
1+\cos\!\left(\pi\frac{i-K+1}{N_f}\right)
\right],
&& K\leq i<K+N_f,\\
\zeta_i &= 0,
&& i\geq K+N_f .
\end{aligned}
\label{eq:residual_fade}
\end{equation}
The sequence $\{\bar r_{k,i}\}_{i=0}^{N-1}$ is held fixed during the
corresponding MPC solve, preserving near-term history dependence without
allowing the residual to vary with optimizer-induced state changes.

Near-term residual predictions benefit from recently observed
histories, whereas recursively predicted histories may
accumulate modeling errors. Restricting residual evaluation
to the first $K$ nodes and smoothly fading the correction
over $N_f$ nodes balances prediction reliability and
computational cost while avoiding abrupt changes in the
predicted dynamics.


}

\begin{figure}[t]
\centering
\includegraphics[width=0.9\linewidth, height=4in]{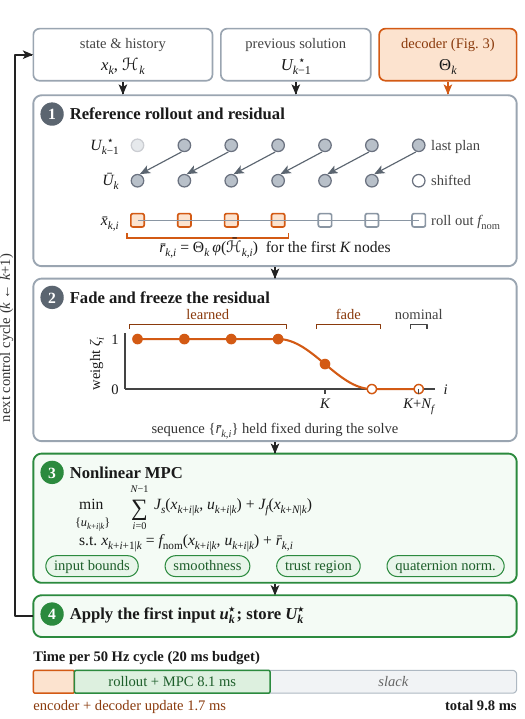}
\caption{
{\color{black}
Integration of the adapted residual model within MPC. The shifted previous
solution defines the reference rollout used to evaluate the near-term
history-dependent residual, which is attenuated and held fixed during the
current optimization.}
}
\label{fig:mpc_pipeline}
\end{figure}

The controller solves
\begin{align}
\min_{\{u_{k+i|k}\}_{i=0}^{N-1}}
\quad &
\sum_{i=0}^{N-1}J_s(x_{k+i|k},u_{k+i|k})
+J_f(x_{k+N|k})
\label{eq:mpc_cost}\\
\mathrm{s.t.}\quad
x_{k+i+1|k}
&=
f_{\mathrm{nom}}(x_{k+i|k},u_{k+i|k})
+\bar r_{k,i},
\label{eq:mpc_dynamics}\\
x_{k|k}&=x_k,\quad
x_{k+i|k}\in\mathcal X,\quad
u_{k+i|k}\in\mathcal U ,
\nonumber
\end{align}
for $i=0,\ldots,N-1$. The costs are
$J_s(x,u)=\|x-x^r\|_Q^2+\|u-u^r\|_R^2$ and
$J_f(x)=\|x-x^r\|_P^2$, where $(x^r,u^r)$ is the reference trajectory and
$Q,R,P\succeq0$.

{\color{black}
The MPC is certainty-equivalent: the decoder posterior
mean is used for prediction, while its covariance supports
Bayesian adaptation and consistency monitoring rather
than uncertainty propagation through the MPC horizon.


Posterior uncertainty could instead be incorporated through
uncertainty-dependent costs, chance constraints, or robust tightening
~\cite{saviolo2023active,yadav2025arcade,hewing2020cautious,tao2024robust}.
The hardware experiments use a $50$\,Hz online loop. On the RTX~4080 host,
encoder inference and decoder adaptation require approximately
$t_{\mathrm{learn}}=1.7$\,ms, while residual-rollout construction and MPC
optimization require $t_{\mathrm{control}}=8.1$\,ms. 
The resulting nominal
computation time is approximately $9.8$\,ms, suggesting feasibility for
update periods as short as $10$\,ms (100 Hz). 
Solver, trust-region, discretization,
and quaternion-normalization details are given in
Appendix~\ref{app:mpc_impl}.
}

\begin{algorithm}[t]
\caption{Online Adaptation and Residual-Aware MPC}
\label{alg:method}
\footnotesize
\begin{algorithmic}[1]
\STATE Observe $x_k$ and compute
$\delta_k=x_k-f_{\mathrm{nom}}(x_{k-1},u_{k-1})$
\STATE Compute $(\lambda_k,\rho_k)$ from $g_{k-1}$ using
\eqref{eq:lambda_adaptive}
\STATE Reopen and update the adaptive decoder using
\eqref{eq:cov_reopening} and \eqref{eq:recursive update}
\STATE Update reference innovations and consistency statistic
$g_k$ using \eqref{err_prediction}--\eqref{eq:ewma_update}
\STATE Form the adapted decoder $\Theta_k$
\STATE Shift $U_{k-1}^{\star}$ and construct
$\{\bar r_{k,i}\}_{i=0}^{N-1}$
\STATE Solve \eqref{eq:mpc_cost}--\eqref{eq:mpc_dynamics} and apply $u_k$
\STATE Form $H_k$, compute $z_k=\phi(H_k)$, and store
$\hat\delta_{k+1|k}=\Theta_kz_k$
\end{algorithmic}
\end{algorithm}

\section{Experimental Validation and Analysis}
\label{sec:experiments}

We evaluate the proposed framework from two complementary perspectives:
(i) residual-model prediction and online adaptation, and
(ii) closed-loop trajectory tracking under payload-induced dynamics changes.

\begin{figure}[t]
\centering
\includegraphics[width=0.98\linewidth]{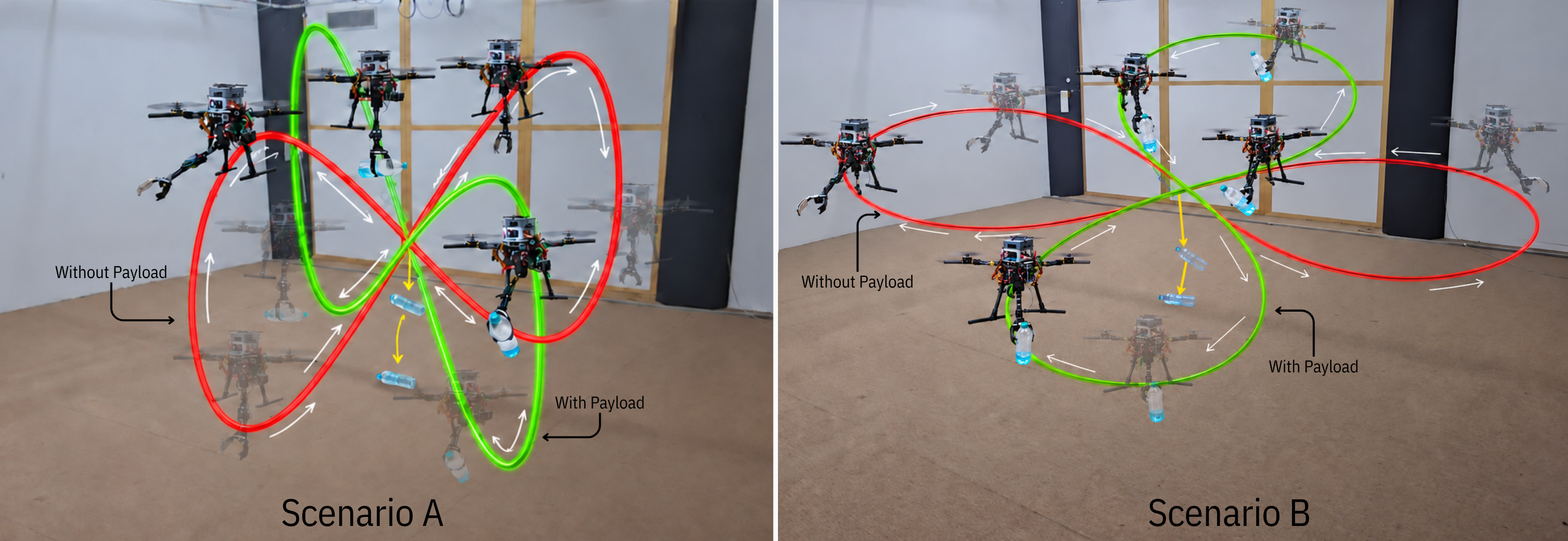}
\caption{\color{black}
Experimental trajectory-tracking scenarios.
Scenario A switches the figure-eight reference from the $xz$- to the
$yz$-plane after payload release, whereas Scenario B reorients its major
axis from the $x$- to the $y$-direction.
}
\label{fig:tracking_scenarios}
\end{figure}

\subsection{Evaluation Protocol}
\label{subsec:evaluation_protocol}

The hardware experiments use the two scenarios in
Fig.~\ref{fig:tracking_scenarios}. In Scenario A, a figure-eight reference
switches from the $xz$- to the $yz$-plane after payload release; in Scenario B,
a horizontal figure-eight changes its major-axis orientation from $x$ to $y$.
The vehicle initially stabilizes at $1$\,m before tracking. Experiments use
$300$ and $500$\,g payloads at a mean platform speed of $0.5$\,m/s, with
manipulator joints commanded within $[-45^\circ,45^\circ]$.

\subsection{Residual Model Validation}
\label{subsec:model_validation}

For residual-model evaluation, the aerial manipulator tracks a smooth
reference using a PID controller for $5$\,min, with a $300$\,g payload
released at $t=2$\,min. The one-step error is
$e^{\mathrm{pred}}_{k+1}=\delta_{k+1}-\Theta_kz_k$.
For channel-wise training-set standard deviations $s_{\delta,j}$, define
$\tilde e^{\mathrm{pred}}_{k,j}
=e^{\mathrm{pred}}_{k,j}/s_{\delta,j}$. We report
\begin{equation} \small
\mathrm{MSE}_{\delta}
=
\frac{1}{N_ed}
\sum_{k=1}^{N_e}\sum_{j=1}^{d}
\left(\tilde e^{\mathrm{pred}}_{k,j}\right)^2,
~\mathrm{RMSE}_{\delta}
=
\sqrt{\mathrm{MSE}_{\delta}},
\label{eq:residual_metrics}
\end{equation}
yielding dimensionless aggregate metrics across residual channels.

\paragraph{Encoder ablation}
We first isolate the contribution of the residual representation. {\color{black}
Four encoders are compared using the same residual target and evaluation
data: \emph{(i)} DNN~\cite{shi2019neural}, using only the current
$(x_k,u_k)$; \emph{(ii)} TCN~\cite{saviolo2022physics}, using temporal
convolutions over $H_k$; \emph{(iii)} temporal-only attention, using the
proposed history representation without cross-variable attention; and
\emph{(iv)} the full structured encoder, using both temporal and
cross-variable attention.
} 
Online decoder adaptation is disabled in this
ablation so that the comparison reflects representation quality.

\begin{table}[t]
\centering
\footnotesize
\renewcommand{\arraystretch}{1.05}
\caption{STANDARDIZED RESIDUAL-PREDICTION RMSE FOR THE ENCODER ABLATION.}
\label{tab:ablation_1}
\begin{tabular}{lcccc}
\toprule
\textbf{Method} &
\textbf{DNN} &
\textbf{TCN} &
\textbf{Temporal only} &
\textbf{Full} \\
\midrule
$\mathrm{RMSE}_{\delta}$ & 0.34 & 0.26 & 0.21 & \textbf{0.18} \\
\bottomrule
\end{tabular}
\end{table}

Table~\ref{tab:ablation_1} shows a consistent improvement as temporal
and cross-variable structure is introduced. The DNN is limited by its
instantaneous input, while the TCN benefits from temporal context.
Attention further improves temporal modeling, and the full encoder gives
the lowest error by additionally modeling dependencies across quadrotor platform,
manipulator, and control variables.

\paragraph{Decoder adaptation ablation}
We next isolate the contribution of online decoder adaptation while keeping
the structured encoder fixed. {\color{black}
Four variants are compared:
\emph{(i)} no update, which retains the offline decoder;
\emph{(ii)} Bayesian update only, which applies the recursive update without
shift-aware covariance adaptation;
\emph{(iii)} fixed forgetting, which uses a constant forgetting factor; and
\emph{(iv)} the proposed adaptive update, which modulates the decoder
covariance using the innovation-consistency monitor.
}

\begin{figure}[t]
\centering
\includegraphics[width=0.9\linewidth]
{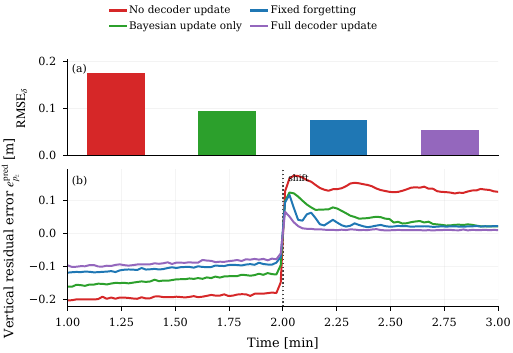}
\vspace{-0.3cm}
\caption{
Decoder adaptation under an abrupt dynamics change at $t=2$\,min.
(a) Standardized residual-prediction RMSE,
$\mathrm{RMSE}_{\delta}$, across all learned residual channels for the four
update strategies.
(b) Representative physical vertical-position residual-prediction error
$e^{\mathrm{pred}}_{p_z}$ before and after the shift.
 }
\label{fig:decoder_ablation}
\end{figure}

Figure~\ref{fig:decoder_ablation} shows that online adaptation substantially
reduces the prediction error after the dynamics change. Bayesian updating
improves over the fixed decoder, while fixed forgetting responds faster at
the cost of a less uniform transient. The proposed consistency-driven update
achieves the lowest overall error and the fastest smooth recovery by
temporarily increasing adaptation when the current decoder becomes
inconsistent with the observed dynamics.

\paragraph{Shift-detection performance}
The consistency monitor is evaluated separately under abrupt payload-release
events to quantify detection reliability and response time.

\begin{table}[t]
\centering
\footnotesize
\caption{Payload-shift detection performance.}
\label{tab:shift_detection}
\renewcommand{\arraystretch}{1.05}
\begin{tabular}{lcccc}
\toprule
\textbf{Payload} &
\textbf{Detection} &
\textbf{False det.} &
\textbf{Missed} &
\textbf{Delay [ms]} \\
\midrule
$200$\,g $\rightarrow 0$\,g & $60\%$  & $10\%$ & $40\%$ & $140$ \\
$300$\,g $\rightarrow 0$\,g & $70\%$  & $0\%$ & $30\%$ & $120$ \\
$400$\,g $\rightarrow 0$\,g & $100\%$  & $20\%$  & $0\%$ & $80$ \\
$500$\,g $\rightarrow 0$\,g & $100\%$ & $10\%$  & $0\%$  & $60$ \\
\bottomrule
\multicolumn{5}{p{0.97\columnwidth}}{
\scriptsize
Each condition contains $10$ trials. A release is detected when
$g_k>1$ within $0.5$\,s of the known event. False detections are
spurious pre-release triggers; delay is the median over successful detections.
}
\end{tabular}
\end{table}

Table~\ref{tab:shift_detection} shows the expected sensitivity--false-alarm
trade-off of the data-driven monitor. Larger payload changes generate stronger
innovation signatures and are detected more reliably and more rapidly,
whereas smaller changes are harder to distinguish from nominal modeling and
measurement variability.

\paragraph{Multi-step prediction and learned rollout depth}
{\color{black}
To assess prediction over the MPC horizon, we evaluate recursive rollouts
on held-out trajectories and vary the number of future prediction steps over
which the history-dependent residual is recomputed. Each rollout starts from
measured state--history data and uses the same recorded control sequence for
all methods; subsequent states are propagated recursively, with future
measurements used only for evaluation.

For multi-step prediction, we report the physical Cartesian-position RMSE,
{  \footnotesize
$\mathrm{RMSE}_{p}(i)
=
\sqrt{
\frac{1}{N_r}
\sum_{r=1}^{N_r}
\left\|
\hat p^{(r)}_{k+i|k}
-
p^{(r)}_{k+i}
\right\|_2^2
}$,   }
where $p=[p_x,p_y,p_z]^\top$ and $N_r$ is the number of evaluated
rollouts.
Figure~\ref{fig:multistep_mpc_ablation}(a) shows that the learned models
substantially limit error accumulation relative to the nominal dynamics over
the prediction horizon. The proposed structured model remains competitive
with the neural baselines under recursive rollout. As shown in
Fig.~\ref{fig:multistep_mpc_ablation}(b), recomputing the history-dependent
residual farther into the horizon improves prediction accuracy but increases
model-evaluation cost. We therefore use a short learned rollout as a practical
accuracy--computation compromise.
}

\begin{figure}[t]
\centering
\includegraphics[width=0.95\columnwidth]{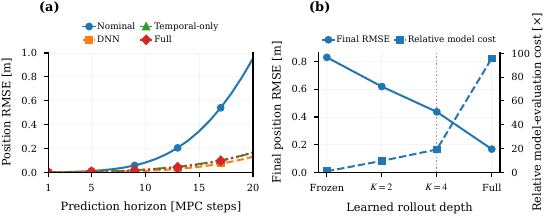}
\vspace{-0.1cm}
\caption{ {\color{black}
Multi-step prediction and learned-rollout ablation.
(a) Recursive platform-position RMSE, $\mathrm{RMSE}_p(i)$, over the
$N=20$ MPC prediction horizon.
(b) Final-horizon position RMSE and relative neural-model evaluation cost versus
learned rollout depth; the dotted line marks the selected $K=4$.
}
}
\label{fig:multistep_mpc_ablation}
\end{figure}

{\color{black}
\paragraph{Latent representation analysis}
We next examine whether the learned feature $z_k=\phi(H_k)$ retains the
temporal, cross-variable, and regime-dependent information targeted by the
encoder design.
For compact notation in Fig.~\ref{fig:latent_analysis}, we group the state
and input variables into aerial-platform and manipulator components as
$x_b=
\begin{bmatrix}
p^\top & v^\top & q^\top & \omega^\top
\end{bmatrix}^{\top},
x_m=
\begin{bmatrix}
q_m^\top & v_m^\top
\end{bmatrix}^{\top}$,
and 
$u_b=
\begin{bmatrix}
f & M^\top
\end{bmatrix}^{\top},
u_m$
denotes the manipulator control input already defined in the system model.
The learned residual is partitioned analogously into
$\delta_b$ and $\delta_m$, denoting the residual-output components associated
with the aerial-platform and manipulator states, respectively.

Figure~\ref{fig:latent_analysis}(a) shows clear regime information in the
latent representation; a linear probe in the original latent space achieves
$98\%$ payload-regime classification accuracy. Temporal perturbations in
Fig.~\ref{fig:latent_analysis}(b) increase prediction error by
approximately $13$--$31\times$ even when older history segments are modified,
confirming that prediction does not depend only on the current sample.
Furthermore, the cross-variable analysis in
Fig.~\ref{fig:latent_analysis}(c) shows dependencies in both directions
between the platform and manipulator residuals. Together, these results provide functional evidence that the nonlinear
encoder captures regime, temporal, and cross-subsystem information before
the lightweight linear decoder.

Unlike the data-level diagnostics in Fig.~\ref{fig:motivation_fig},
these perturbations assess the information retained by
the learned representation.
Fig.~\ref{fig:latent_analysis} shows that the learned representation itself
uses temporal and cross-subsystem information.
}

\begin{figure}[t]
\centering
\includegraphics[width=\columnwidth]
{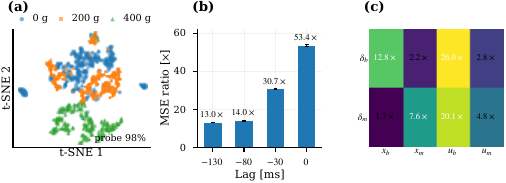}
\vspace{-0.7cm}
\caption{ {\color{black}
Learned latent representation and functional perturbation analysis on
held-out data.
(a) t-SNE visualization of the latent vectors $z_k=\phi(H_k)$ for the
$0$, $200$, and $400$\,g payload regimes. The reported $98\%$ linear-probe
accuracy is evaluated in the original latent space, not in the t-SNE
projection.
(b) Temporal permutation analysis. The plotted MSE ratio is
$\mathrm{MSE}_{\mathrm{pert}}/\mathrm{MSE}_{\mathrm{base}}$,
where both quantities are the standardized residual MSE defined in
\eqref{eq:residual_metrics}. For each perturbation, a complete history block
is replaced by the corresponding block from another held-out sample,
breaking its alignment with the current target while preserving the
within-block multivariate structure. The displayed lags $-130$, $-80$, and
$-30$\,ms denote the centers of the intervals
$[-150,-110]$, $[-100,-60]$, and $[-50,-10]$\,ms, respectively; $0$\,ms
denotes the current sample.
(c) Cross-variable permutation analysis using the same MSE ratio.
Here, $x_b$, $x_m$, $u_b$, and $u_m$ denote the aerial-platform state,
manipulator state, aerial-platform input, and manipulator input groups,
respectively, while $\delta_b$ and $\delta_m$ denote the corresponding
platform and manipulator residual-output groups. Values greater than one
indicate increased prediction error after removing the aligned information
from the indicated input group.
}
}
\label{fig:latent_analysis}
\end{figure}

\paragraph{History-length ablation}

{\color{black}
Figure~\ref{fig:history_ablation} evaluates the trade-off
between history length, prediction accuracy, inference
latency, and post-transition adaptation. The selected
$h=15$ achieves the lowest validation and post-transition
errors among the evaluated settings while maintaining
low inference latency.
}


\begin{figure}[t]
\centering
\includegraphics[width=0.95\columnwidth]
{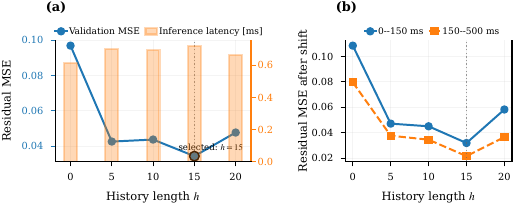}
\vspace{-0.1cm}
\caption{ {\color{black}
History-length ablation.
(a) Validation standardized residual MSE,
$\mathrm{MSE}_{\delta}$, and batch-1 encoder--decoder inference latency
versus history length.
(b) Median standardized residual MSE across transition rollouts during the
immediate $0$--$150$\,ms and subsequent $150$--$500$\,ms intervals.
}
}
\label{fig:history_ablation}
\end{figure}

\subsection{Closed-Loop Trajectory Tracking Performance}
\label{subsec:tracking}

Closed-loop performance is quantified over
$e_k^{\mathrm{trk}}
=[e_x,e_y,e_z,e_{\psi},e_{q_1},e_{q_2}]^\top$,
comprising Cartesian platform position, yaw, and the two manipulator-joint
angles. To combine quantities with different physical units, each channel is
normalized by its fixed evaluation scale $s_j^{\mathrm{trk}}$,
$\tilde e_{k,j}^{\mathrm{trk}}
=e_{k,j}^{\mathrm{trk}}/s_j^{\mathrm{trk}}$.
The reported scalar metric is
$\mathrm{RMSE}_{\mathrm{trk}}
=
\sqrt{
\frac{1}{6N_t}
\sum_{k=1}^{N_t}
\sum_{j=1}^{6}
\left(\tilde e_{k,j}^{\mathrm{trk}}\right)^2
}$,
which is dimensionless. The normalization scale $s_j^{\mathrm{trk}}$ is the training-set standard
deviation of the corresponding tracked channel and is fixed for all methods
and test conditions. Figure~\ref{fig:trajectory_generalization}
additionally reports representative position and joint-angle errors in their
physical units.
We next evaluate closed-loop tracking in Scenarios A and B using the
normalized six-channel tracking RMSE under the unseen
{\color{black}interpolated} $300$\,g payload and the
{\color{black}extrapolative} $500$\,g payload. {\color{black}
The compared methods span no-learning, offline learned-dynamics, and
online-adaptive approaches. Nominal MPC provides the physics-only reference;
DNN~\cite{shi2019neural} and Diffusion~\cite{das2025dronediffusion}
represent offline learned residual models; Proto-MPC~\cite{gu2024proto} and
Active MLP~\cite{saviolo2023active} provide online-adaptive comparisons.
}

\begin{figure}[t]
\centering
\includegraphics[width=0.9\columnwidth]{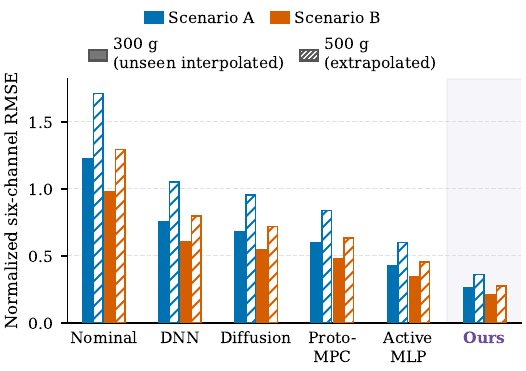}
\vspace{-0.2cm}
\caption{
Closed-loop normalized six-channel tracking RMSE,
$\mathrm{RMSE}_{\mathrm{trk}}$, in Scenarios A and B for the unseen
interpolated $300$\,g payload and extrapolative $500$\,g payload.
The metric combines platform position $(x,y,z)$, yaw $\psi$, and manipulator
joint angles $(q_1,q_2)$ after channel-wise normalization and is therefore
dimensionless.
}
\label{fig:tracking_result}
\end{figure}

As shown in Fig.~\ref{fig:tracking_result}, tracking accuracy improves
consistently from the nominal and offline learned models to the
online-adaptive methods, with the proposed framework achieving the lowest
normalized six-channel RMSE across both scenarios and payload conditions.
Because the metric jointly evaluates platform position, yaw, and manipulator
joint tracking after channel-wise normalization, the comparison reflects the
overall commanded aerial-manipulator motion rather than platform position
alone.

\begin{figure*}[t]
\centering
\begin{minipage}[t]{0.49\linewidth}
    \centering
    \includegraphics[width=\linewidth, height=3.25 in]
    {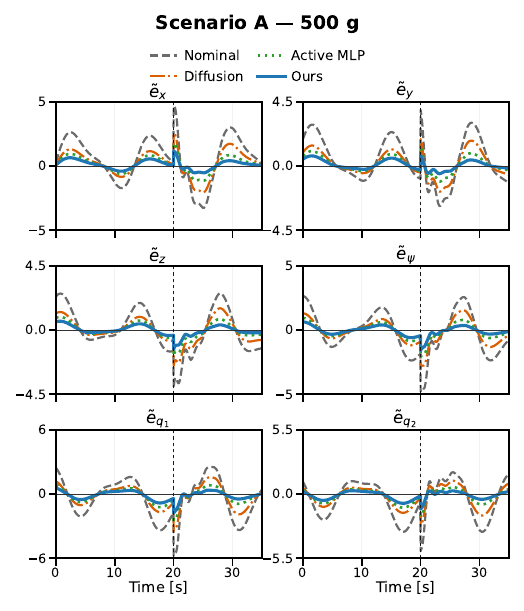}
\end{minipage}
\hfill
\begin{minipage}[t]{0.49\linewidth}
    \centering
    \includegraphics[width=\linewidth, height=3.25 in]
    {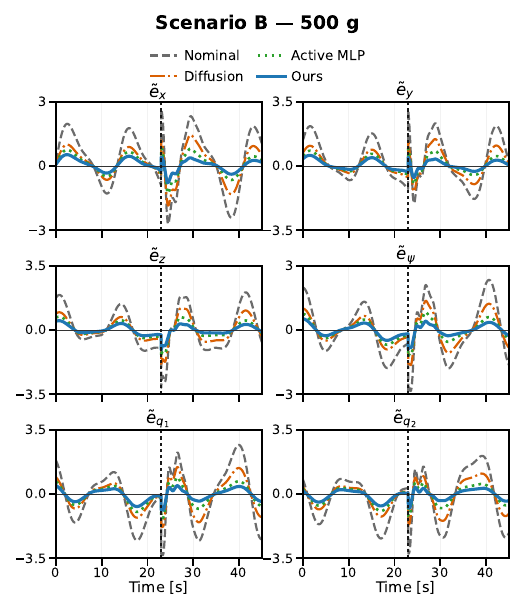}
\end{minipage}
\vspace{-0.4cm}
\caption{
{\color{black}Normalized state-wise tracking errors for the extrapolative $500$\,g
payload in Scenarios A and B. The dashed vertical lines indicate payload
release at $20$\,s and $23$\,s, respectively. Errors correspond to platform
position $(x,y,z)$, yaw $\psi$, and manipulator joints $(q_1,q_2)$ using the
same channel normalization as $\mathrm{RMSE}_{\mathrm{trk}}$ in
Fig.~\ref{fig:tracking_result}.
}}
\label{fig:tracking_transients}
\end{figure*}

{\color{black}
Figure~\ref{fig:tracking_transients} further resolves the aggregate result
into its six tracked channels for the more challenging $500$\,g condition.
All methods experience a transient at payload release; however, the fixed
nominal and offline learned models retain larger post-release errors, whereas
online adaptation progressively restores tracking. The proposed method shows
the smallest release transient and fastest recovery across the platform and
manipulator channels, consistent with the aggregate trend in
Fig.~\ref{fig:tracking_result}. 
}



\subsection{Trajectory generalization}
{\color{black}
To further evaluate the performance, we additionally
consider three trajectories. A
$500$\,g payload is released near the midpoint of each trajectory to introduce
an abrupt dynamics change.
Figure~\ref{fig:trajectory_generalization} shows stable platform and
manipulator tracking across all three trajectory families. Payload release
produces a short transient in both position and joint errors, after which the
tracking error rapidly returns toward its pre-transition level. These results
show that the observed performance is not specific to the original
figure-eight trajectories.
}

\begin{figure*}[t]
\centering
\includegraphics[width=0.98\linewidth, height=2 in]{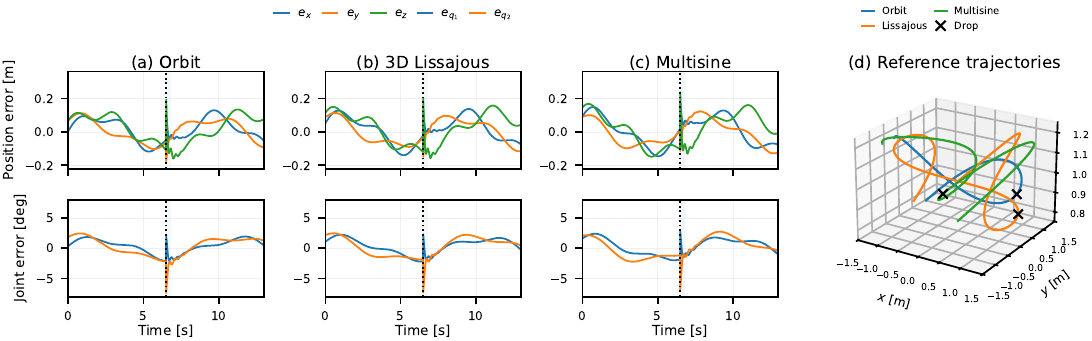}
\vspace{-0.3cm}
\caption{ {\color{black}
Tracking performance across additional trajectory shapes for the extrapolative $500$\,g
payload.
(a)--(c) Platform position and manipulator-joint tracking errors for three different trajectories.
(d) Corresponding reference trajectories; markers indicate the payload-release
locations. The plotted errors are
$e_p=p-p^r=[e_x,e_y,e_z]^\top$ and
$e_{q_m}=q_m-q_m^r=[e_{q_1},e_{q_2}]^\top$;
joint errors are displayed in degrees. The dashed vertical lines indicate payload
release instances.
}
}
\label{fig:trajectory_generalization}
\end{figure*}

\section{Conclusion}

We presented a structured residual model for aerial-manipulator dynamics that
combines a nonlinear history encoder with a recursively adapted linear latent
decoder. The encoder captures temporal and cross-variable coupling, while
innovation consistency reopens decoder covariance when operating conditions
change. Integrated with MPC, the approach improves one- and multi-step
prediction, adapts effectively to payload changes, and reduces closed-loop
tracking error relative to the evaluated offline and online-adaptive
baselines.

{\color{black}
Performance remains limited by the operating region represented by the
frozen encoder and by the local additive-residual approximation. Moreover,
the consistency monitor is empirically calibrated and the controller is
certainty-equivalent; we therefore make no analytical claims on detector
delay, decoder convergence under arbitrary switching, or uncertainty-aware
closed-loop guarantees. Future work will address uncertainty propagation,
broader out-of-distribution adaptation, and contact-rich manipulation.
}

\section*{\small Declaration of AI Assistance} {\small
ChatGPT was used during manuscript language refinement. All scientific claims, implementation, experiments, results,
and final verification were reviewed and approved by the authors.
}

\appendix
\section{Experimental and Implementation Details}
\label{app:exp}

\subsection{Estimator Properties and Scope}
\label{app:theory}

Under Assumptions~\ref{ass:bounded_features}--\ref{ass:realizability},
we establish covariance-update validity and a one-step prediction-error
bound for the recursive decoder. Global identification, convergence
under arbitrary switching, and MPC stability or recursive feasibility
are not guaranteed.

\subsubsection{Recursive Bayesian Update}

For each output $j$, Section~\ref{sec:bayesian_update} is a
linear-Gaussian recursive parameter estimator with $\theta_{k,j}$ as the
estimated state~\cite{simon2006optimal,ljung1999system}.

\begin{lemma}[Well-posedness and covariance validity]
\label{lem:bayesian_validity}
If $P_{k|k-1,j}\succeq0$ and $\sigma_j^2>0$, then
$S_{k,j}=z_{k-1}^\top P_{k|k-1,j}z_{k-1}+\sigma_j^2>0$, and the
Joseph update in \eqref{eq:recursive update} yields
$P_{k|k,j}\succeq0$.
\end{lemma}

\begin{proof}
Positive semidefiniteness gives
$z_{k-1}^\top P_{k|k-1,j}z_{k-1}\geq0$, hence $S_{k,j}>0$.
Writing the posterior as
$P_{k|k,j}=AP_{k|k-1,j}A^\top+
\sigma_j^2K_{k,j}K_{k,j}^\top$ with
$A=I-K_{k,j}z_{k-1}^\top$ shows that both terms are positive semidefinite.
The same argument holds for the robustified update because
$\sigma_{k,j,\mathrm{eff}}^2\geq\sigma_j^2>0$.
\end{proof}

\subsubsection{One-Step Prediction Error}

\begin{proposition}[One-step residual-prediction error]
\label{prop:prediction_error}
Under Assumptions~\ref{ass:bounded_features} and
\ref{ass:realizability},
\begin{equation}
\left|
\delta_{k+1}^{(j)}-z_k^\top\hat\theta_{k,j}
\right|
\leq
c_z\|\theta_j^\star-\hat\theta_{k,j}\|_2
+\bar\xi_j+|\varepsilon_k^{(j)}|.
\label{eq:prediction_error_bound}
\end{equation}
\end{proposition}

\begin{proof}
Substituting \eqref{eq:local_latent_approx} and applying the triangle and
Cauchy--Schwarz inequalities gives
$|\delta_{k+1}^{(j)}-z_k^\top\hat\theta_{k,j}|
\leq\|z_k\|_2\|\theta_j^\star-\hat\theta_{k,j}\|_2+
|\xi_{k,j}|+|\varepsilon_k^{(j)}|$; the result follows from the two
assumptions.
\end{proof}

Equation~\eqref{eq:prediction_error_bound} separates decoder-estimation,
representation, and stochastic errors but does not imply parameter
convergence.

\subsubsection{Covariance Reopening and Consistency Monitoring}

From \eqref{eq:cov_reopening},
$P_{k|k-1,j}-P_{k-1|k-1,j}=\rho_kP_{0,j}\succeq0$ for
$\rho_k\geq0$ and $P_{0,j}\succeq0$; covariance reopening therefore
reduces confidence in the previous estimate without violating positive
semidefiniteness~\cite{ljung1999system,simon2006optimal}.
No convergence claim is made under arbitrary switching.
The clipped EWMA in \eqref{eq:ewma_update} geometrically weights recent
innovation scores~\cite{roberts1959control}. Here it is an empirically
calibrated consistency indicator, not a hypothesis test; false-alarm and
detection-delay performance are therefore evaluated experimentally in
Table~\ref{tab:shift_detection}.

\subsection{Hardware Setup and Data Collection}

Experiments use a Tarot-650 quadrotor with SunnySky V4006 motors, a 6S LiPo
battery, 14-inch propellers, and a two-DoF manipulator driven by Dynamixel
XM430-W210-T motors. A CUAV X7+ running customized PX4 provides flight
control, with a Jetson Orin Nano Super companion computer communicating via
Micro XRCE-DDS; manipulator actuation uses \textit{ros2\_control} in
current-based torque mode. Platform state is obtained from a $120$\,fps
OptiTrack system fused with onboard IMU measurements, and joint states from
Dynamixel encoders. Learned inference and high-level control run offboard on
an RTX~4080 Laptop GPU with Intel i9-13900HX CPU over WiFi, while low-level
control remains onboard. Reported computation times refer to this host.
Hardware results use $10$ trials per evaluated condition.


Using a PID controller, randomized smooth trajectories excite platform
position, velocity, attitude, manipulator configuration, and control inputs,
including payload pickup/release events. For each $0$, $200$, and $400$\,g
payload regime, $5$\,min of synchronized data are recorded at $100$\,Hz:
$\mathcal D=\{(x(t_i),u(t_i))\}_{i=1}^{N_{\mathrm{data}}}$.
These data generate the histories and residual targets of
Section~\ref{sec:problem_formulation}.

\subsection{Nominal Aerial-Manipulator Model}

The nominal model retains platform and manipulator rigid-body dynamics while
neglecting their dynamic coupling:
\begin{align} \footnotesize
\dot p &= v,
&
\dot v &= \frac{1}{m} f z_B + g_I,
\nonumber\\
\dot q
&=
\frac{1}{2}q\otimes
\begin{bmatrix}
0\\ \omega
\end{bmatrix},
&
\dot\omega
&=
J^{-1}\!\left(M-\omega\times J\omega\right),
\nonumber\\
\dot q_m &= v_m,
&
\dot v_m &= B(u_m-g_m).
\label{eq:nominal_continuous}
\end{align}
Here $p,v,\omega\in\mathbb R^3$, $q\in\mathbb S^3$, and
$q_m,v_m\in\mathbb R^2$; $z_B$ is the inertial-frame thrust direction and
$g_I=[0,0,-g_0]^\top$. The manipulator uses
$B=\mathrm{diag}(J_1^{-1},J_2^{-1})$ and
$g_m=[\gamma_1\cos q_1+\gamma_2\cos(q_1+q_2),
\gamma_2\cos(q_1+q_2)]^\top$, with
$\gamma_1=(m_1l_{c1}+m_2l_1)g_0$ and $\gamma_2=m_2l_{c2}g_0$.
The input and state are
$u=[f,M^\top,u_m^\top]^\top\in\mathbb R^6$ and
$x=[p^\top,v^\top,q^\top,\omega^\top,q_m^\top,v_m^\top]^\top
\in\mathbb R^{17}$.
Runge--Kutta integration yields
$x_{k+1}=f_{\mathrm{nom}}(x_k,u_k)$; coupling and remaining unmodeled
effects are represented by the learned additive residual. The nominal model parameters are summarized in
Table~\ref{tab:nominal_parameters}.

\begin{table}[t]
\centering
\footnotesize
\renewcommand{\arraystretch}{0.90}
\caption{Nominal aerial-manipulator parameters.}
\label{tab:nominal_parameters}
\scalebox{0.90}{
\begin{tabular}{lll}
\toprule
\textbf{Symbol} & \textbf{Description} & \textbf{Value} \\
\midrule
$m$      & Platform mass       & $2.2$ kg \\
$J$      & Nominal body inertia matrix & $\mathrm{diag}(0.10,\;0.10,\;0.06)$ kg\,m$^2$ \\
$g_0$    & Gravity             & $9.81$ m/s$^2$ \\
$m_1$    & Link-1 mass         & $0.22$ kg \\
$m_2$    & Link-2 mass         & $0.30$ kg \\
$l_1$    & Link-1 length       & $0.12$ m \\
$l_2$    & Link-2 length       & $0.16$ m \\
$l_{c1}$ & Link-1 COM          & $0.05$ m \\
$l_{c2}$ & Link-2 COM          & $0.10$ m \\
$J_1$    & Joint-1 inertia     & $8.36\times10^{-4}$ kg\,m$^2$ \\
$J_2$    & Joint-2 inertia     & $3.76\times10^{-3}$ kg\,m$^2$ \\
\bottomrule
\end{tabular}}
\end{table}

\subsection{Learning and Online Adaptation}
\label{app:decoder_online}

The JAX model is trained with Adam using
\eqref{eq:offline_objective}; encoder $\phi$ and decoder $\Theta_0$ are
learned jointly, after which only the decoder is updated online.
The offline architecture and training settings are summarized in
Table~\ref{tab:offline_training}, while the online adaptation parameters are
listed in Table~\ref{tab:online_adaptation}.

\begin{table}[t]
\centering
\footnotesize
\renewcommand{\arraystretch}{0.90}
\caption{Offline training configuration.}
\label{tab:offline_training}
\scalebox{0.90}{
\begin{tabular}{lc}
\toprule
\textbf{Parameter} & \textbf{Value} \\
\midrule
Encoder layers & $3$ \\
History length $h$ & $15$ \\
Segment length $L_{\mathrm{seg}}$ & $5$ \\
Embedding dimension $d_{\mathrm{model}}$ & $64$ \\
Attention heads & $4$ \\
Feedforward dimension & $64$ \\
Projection & $64\!\rightarrow\!32\!\rightarrow\!16$; ELU \\
Latent dimension $\ell$ & $16$ \\
Optimizer & Adam \\
Learning rate & $10^{-3}$ \\
Batch size & $256$ \\
Epochs & $100$ \\
Weight decay & $10^{-5}$ \\
\bottomrule
\end{tabular}}
\end{table}

Initial Bayesian statistics are calibrated from offline training features
$Z_{\mathrm{tr}}$. Specifically,
$G_{\mathrm{tr}}=N_{\mathrm{tr}}^{-1}Z_{\mathrm{tr}}^\top Z_{\mathrm{tr}}
+\epsilon_GI$ and
$P_0=p_0\ell G_{\mathrm{tr}}^{-1}/
\mathrm{tr}(G_{\mathrm{tr}}^{-1})$, with
$P_{0|0,j}=P_0$ for all $d=17$ residual outputs.
For offline decoder error
$r_{k,j}=\delta_{k+1}^{(j)}-z_k^\top\theta_{0,j}$,
the observation variance is robustly estimated as
$\sigma_j^2=\max[(1.4826\,\mathrm{MAD}\{r_{k,j}\})^2,
\sigma_{\min}^2]$.
Consistency monitoring uses the clipped EWMA of
Section~\ref{sec:shift_detection}; thresholds $\eta_j$ are calibrated from
nominal training data and a mismatch requires at least $n_{\min}$ coherent
channels. A non-inflated reference decoder is used only for this monitor.

For adaptive-decoder innovation
$\nu_{k,j}=|\delta_k^{(j)}-z_{k-1}^\top\mu_{k|k-1,j}|/\sqrt{S_{k,j}}$,
the Huber weight is $\omega_{k,j}=1$ for $\nu_{k,j}\leq c_H$ and
$\omega_{k,j}=\max(\omega_{\min},c_H/\nu_{k,j})$ otherwise.
The update uses
$\sigma_{k,j,\mathrm{eff}}^2=\sigma_j^2/\omega_{k,j}^2$
in place of $\sigma_j^2$, reducing the influence of extreme innovations.
The covariance spectrum is bounded according to Table~\ref{tab:online_adaptation}.

\begin{table}[t]
\centering
\footnotesize
\renewcommand{\arraystretch}{0.90}
\caption{Online adaptation parameters.}
\label{tab:online_adaptation}
\scalebox{0.90}{
\begin{tabular}{lc}
\toprule
\textbf{Parameter} & \textbf{Value} \\
\midrule
Initial covariance $P_{0|0,j}$ &
Offline Gram, $p_0=0.1$ \\
Observation variance $\sigma_j^2$ &
Offline MAD; median $2.5\times10^{-3}$ \\
Huber threshold $c_H$ & $12$ \\
Minimum Huber weight $\omega_{\min}$ & $0.08$ \\
EWMA coefficient $\alpha$ & $0.5$ \\
Innovation clipping $c_{\mathrm{clip}}$ & $5$ \\
Threshold $\eta_j$ & $97.5$th percentile/channel \\
Minimum coherent channels $n_{\min}$ & $5$ \\
Adaptation gain $\beta$ & $0.5$ \\
Minimum scaling $\lambda_{\min}$ & $0.985$ \\
Reopening gain $\gamma$ & $0.5$ \\
Consistency cap $c_g$ & $1.0$ \\
Eigenvalue floor & $10^{-10}$ \\
Eigenvalue ceiling & $1.10\,\lambda_{\max}(P_0)$ \\
Gram regularization ($\epsilon_G$) & $10^{-3}$ \\
Observation-variance floor ($\sigma_{\min}^2$) & $10^{-5}$ \\
\bottomrule
\end{tabular}}
\end{table}

{\color{black}
\subsection{Baseline Configuration}

All baselines use the same nominal aerial-manipulator model, state--input
variables, residual targets, tracking objective, constraints, and evaluation
trajectories while retaining their original learned representations and
adaptation mechanisms. For Proto-MPC, its quadrotor nominal model is replaced
by the same aerial-manipulator nominal dynamics used here, while its
prototype-decoder mechanism is unchanged.
}

\subsection{MPC Implementation}
\label{app:mpc_impl}

The nonlinear MPC is implemented in JAX and solved with
\texttt{Optimistix}~\cite{optimistix} using direct shooting.  The MPC parameters are summarized in Table~\ref{tab:mpc_config}.
{\color{black}
The shifted
previous optimum both warm-starts the solver and defines the reference
rollout for residual evaluation. An element-wise trust region enforces
$|u_{k+i|k}-\bar u_{k,i}|\preceq\Delta u_{\mathrm{tr}}$ together with the
physical input limits; control-smoothness and soft-constraint penalties are
also used for numerical robustness.
The residual is recomputed for $K=4$ prediction nodes and then attenuated as
in Section~\ref{sec:mpc}. The MPC grid uses
$\Delta t_{\mathrm{MPC}}=0.05$\,s and $N=20$, while replanning occurs at
$50$\,Hz. After each propagated step, the quaternion is normalized as
$q_{k+i+1|k}=\tilde q_{k+i+1|k}/\|\tilde q_{k+i+1|k}\|_2$. }

\begin{table}[t]
\centering
\footnotesize
\renewcommand{\arraystretch}{0.92}
\caption{MPC configuration; input entries follow
$u=[f,M^\top,u_m^\top]^\top$.}
\label{tab:mpc_config}
\scalebox{0.75}{
\begin{tabular}{lc}
\toprule
\textbf{Parameter} & \textbf{Value} \\
\midrule
$Q$ &
$\mathrm{diag}([50,50,80,10,10,10,1,5,5,5,1,1,1,20,20,5,5])$ \\
$P$ & $4Q$ \\
$R$ & $\mathrm{diag}([10,20,20,20,10,10])$ \\
Input lower bound & $[0.1,-0.5,-0.5,-0.5,-1,-1]$ \\
Input upper bound & $[20,0.5,0.5,0.5,1,1]$ \\
Initialization & Hover thrust + zero torques \\
Prediction nodes $N$ & $20$ \\
Prediction horizon & $1$\,s \\
Prediction interval $\Delta t_{\mathrm{MPC}}$ & $0.05$\,s \\
Learned depth $K$ & $4$ nodes \\
Residual fade & $N_f=2$ nodes ($0.10$\,s) \\
Trust radius $\Delta u_{\mathrm{tr}}$ &
$[3.0,0.225,0.225,0.080,0.175,0.175]$ \\
\bottomrule
\end{tabular}}
\end{table}

\bibliographystyle{IEEEtran}
\bibliography{root}

\begin{IEEEbiography}
[{\includegraphics[width=1in,height=1.25in,clip,keepaspectratio]{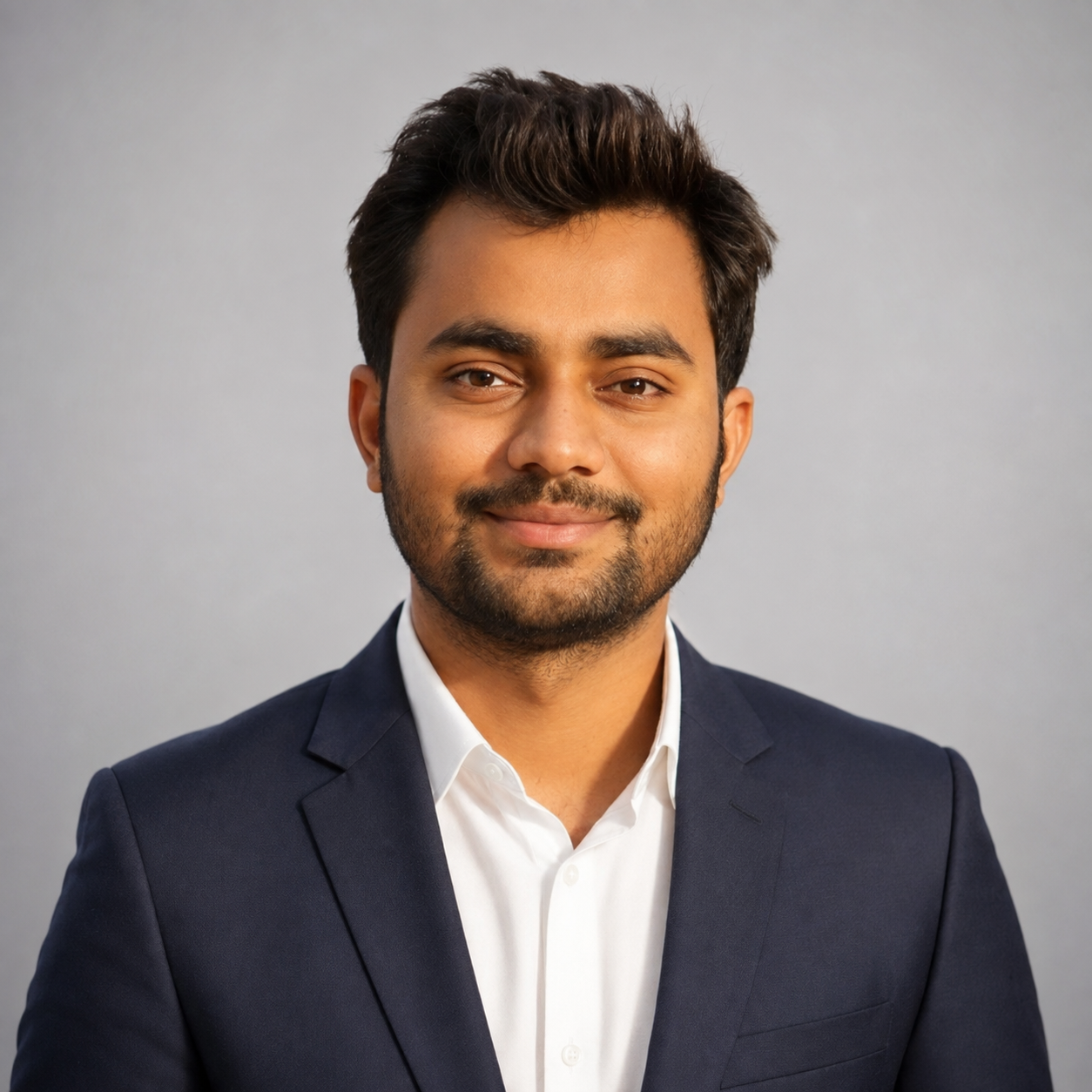}}]
{Rishabh Dev Yadav}
received the B.Tech. degree in Mechanical Engineering from the
IIIT Jabalpur, India, in 2019,
and the M.S. by Research degree in Computer Science and Engineering
from the IIIT Hyderabad,
India, in 2023. He is pursuing the Ph.D. degree in Computer Science
at The University of Manchester, UK, and is an ELLIS Ph.D. student.
He has been a Visiting Ph.D. Researcher at TU Delft, The Netherlands.
His research interests include aerial robotics, dynamics learning,
and learning-based adaptive control.
\end{IEEEbiography}

\begin{IEEEbiography}[{\includegraphics[width=1in,height=1.25in,clip,keepaspectratio]{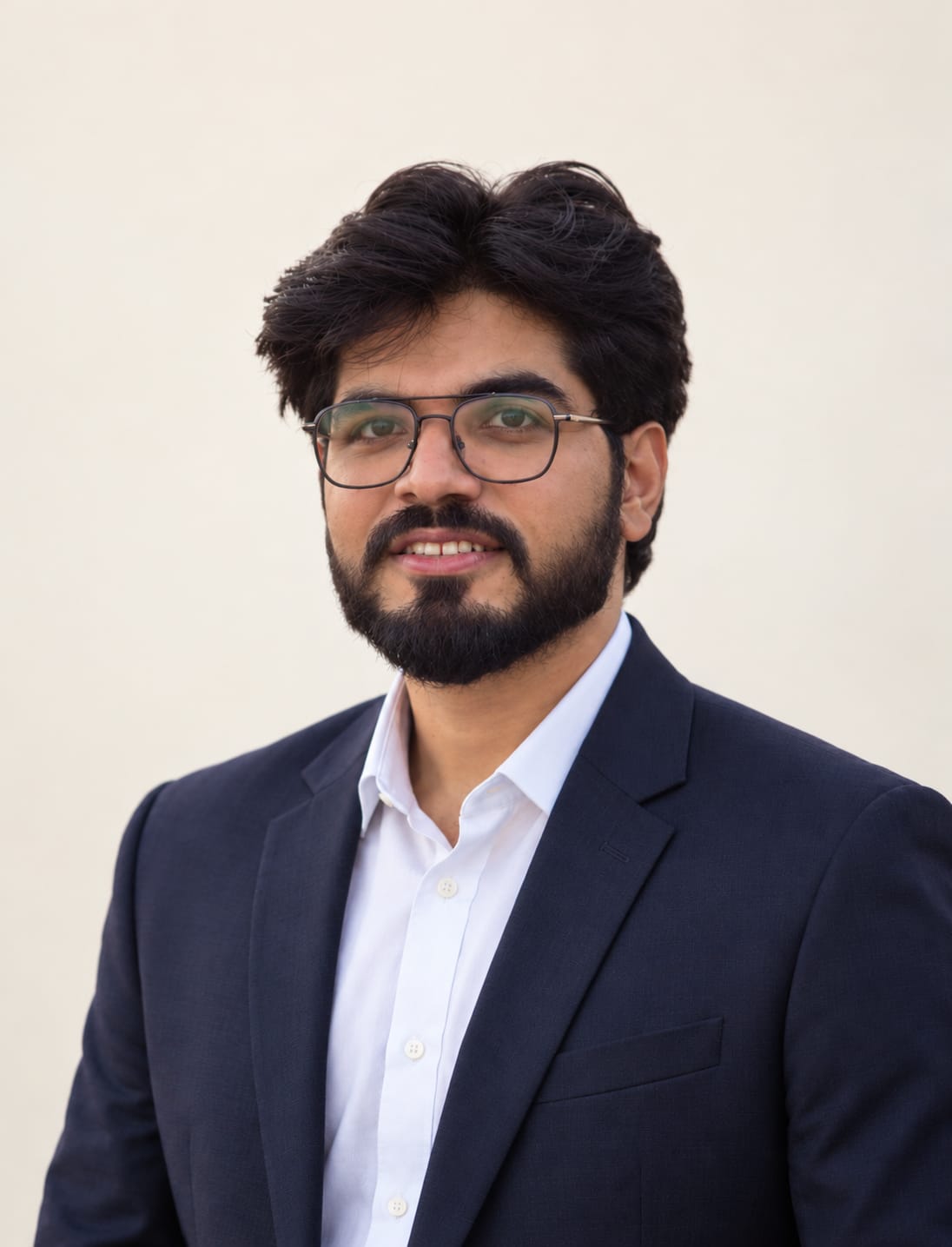}}]{Samaksh Ujjawal}  received the B.Tech. degree in Mechanical Engineering from Bundelkhand Institute of Engineering and Technology, India, in 2023. He is pursuing the M.S. by Research degree in Robotics at IIIT Hyderabad, India. His research focuses on residual dynamics learning, real-time embedded deployment, reinforcement learning,
sim-to-real transfer, and aerial manipulation.
\end{IEEEbiography}

\begin{IEEEbiography}
[{\includegraphics[width=1in,height=1.25in,clip,keepaspectratio]{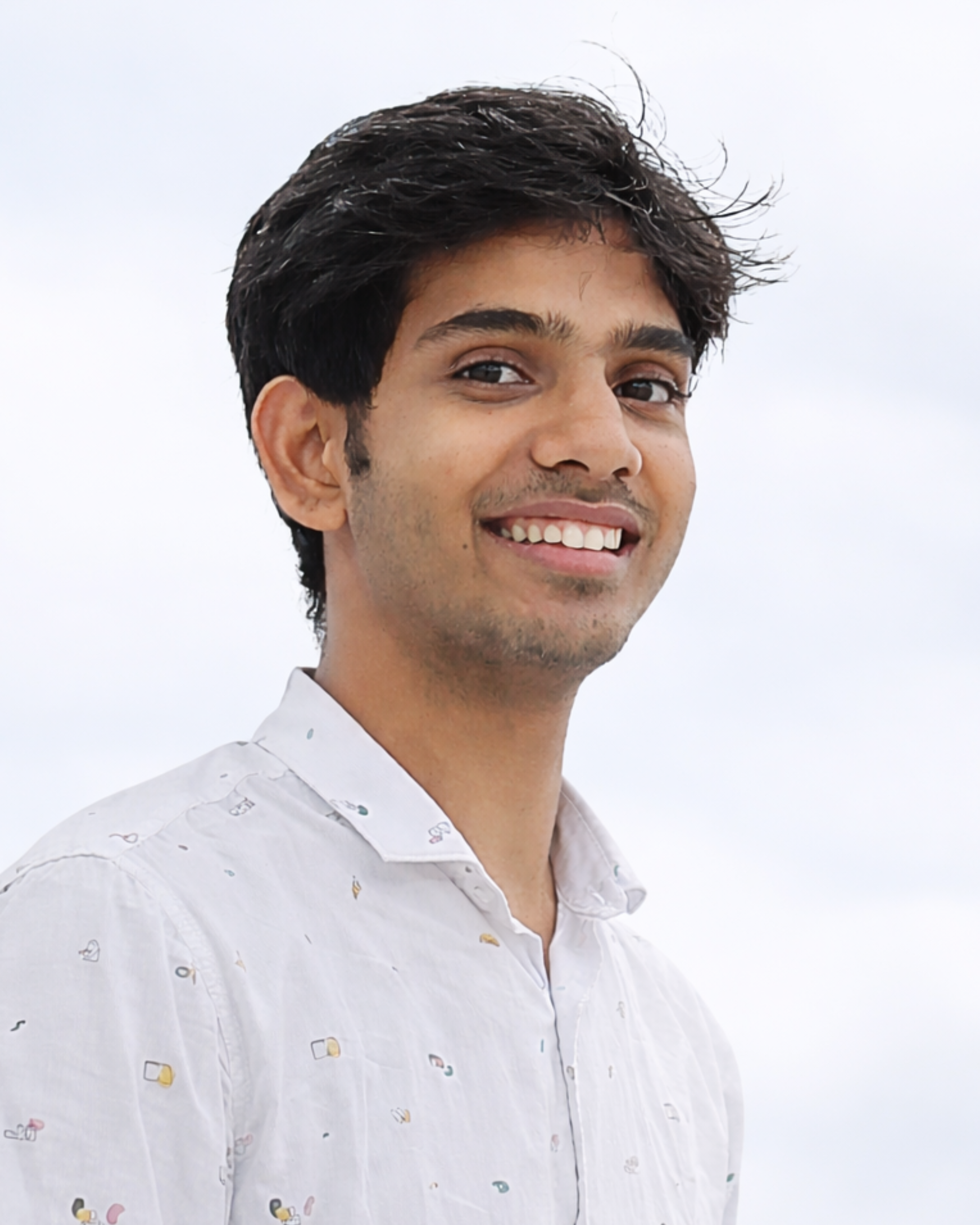}}]{Shivansh Pratap Singh} received the B.Tech. degree in Mechanical Engineering from Bundelkhand Institute of Engineering and Technology, Jhansi, India, in 2023. He is currently pursuing the master’s degree in Electronics and Communication Engineering at the International Institute of Information Technology (IIIT) Hyderabad, India. His research interests include learning-based control, control policy optimization, adaptive control, aerial manipulators, and robot learning.
	\end{IEEEbiography}

\begin{IEEEbiography}
[{\includegraphics[width=1in,height=1.25in,
clip,keepaspectratio]{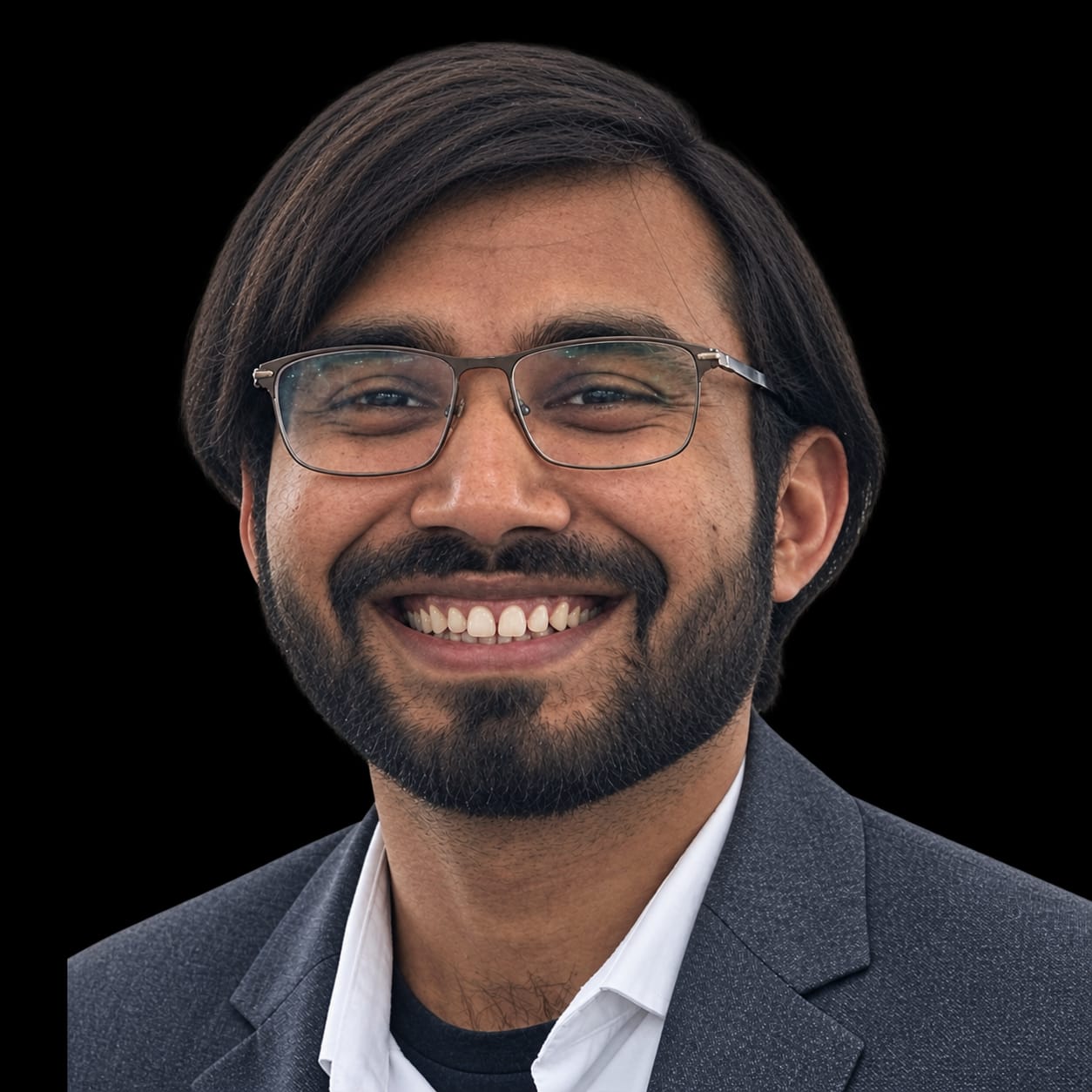}}]
{Akshat Singh}
received the B.S.-M.S. dual degree from IISER Bhopal, India, in 2025. He is pursuing the Ph.D. degree at Instituto Superior T\'ecnico, University of Lisbon, Portugal, with the Institute for Systems and Robotics.
His research focuses on extended reality, human--swarm interaction, aerial robotics, and multi-robot systems.
\end{IEEEbiography}

\begin{IEEEbiography}
[{\includegraphics[width=1in,height=1.25in,clip,keepaspectratio]{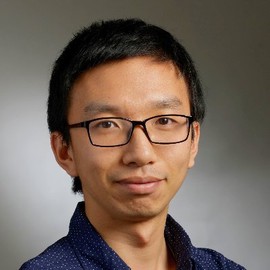}}]
{Sihao Sun}
received the B.Sc. and M.Sc. degrees from Beihang University,
China, in 2014 and 2017, respectively, and the Ph.D. degree
from TU Delft, The Netherlands, in 2020, all in Aerospace
Engineering. From 2020 to 2021, he was a Visiting Scholar
and Postdoctoral Researcher with the Robotics and Perception
Group, University of Zurich, Switzerland. He is currently
with the Department of Cognitive Robotics at TU Delft.
His research interests include system identification,
aerial robotics, aerial manipulation, and nonlinear control.
\end{IEEEbiography}

\begin{IEEEbiography}
[{\includegraphics[width=1in,height=1.25in,clip,keepaspectratio]{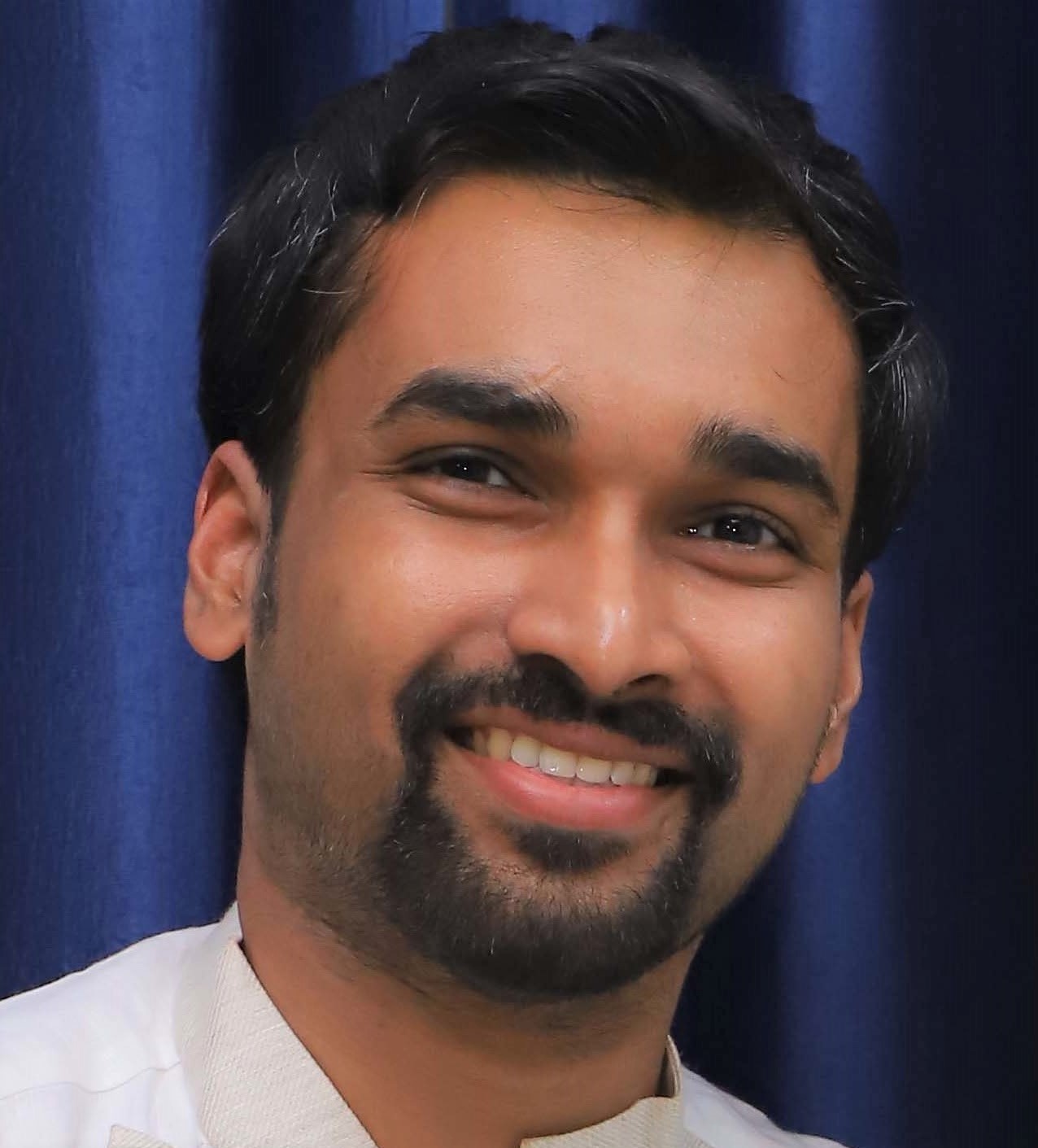}}]
{Spandan Roy}
received the B.Tech. degree in 2011, the M.Tech. degree
in 2013, and the Ph.D. degree in Control and Automation
from IIT Delhi, India, in 2018. He is an Assistant Professor
at IIIT Hyderabad, India, and was previously a Postdoctoral
Researcher at TU Delft, The Netherlands. His research
interests include adaptive-robust control, switched systems,
and Euler--Lagrange robotics. He serves as a Subject Editor
of the International Journal of Adaptive Control and Signal
Processing, an Associate Editor of IEEE Control Systems
Letters, and a Technical Editor of IEEE/ASME TMECH.
\end{IEEEbiography}

\begin{IEEEbiography}
[{\includegraphics[width=1in,height=1.25in,clip,keepaspectratio]{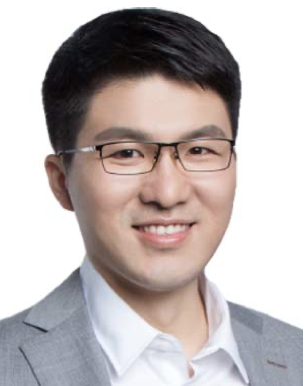}}]
{Wei Pan}
received the Ph.D. degree in Bioengineering from Imperial College
London, UK, in 2016. He is a Professor of Autonomous Systems and
Automatic Control at Newcastle University, UK. His research interests
include robot learning and control, Bayesian machine learning, and
dynamical systems. He is a Senior Editor of IEEE Robotics and
Automation Letters and ACM Transactions on Probabilistic Machine
Learning, and an Associate Editor or Area Chair for IEEE T-RO,
T-RL, T-ASE, ICRA, IROS, RSS, CoRL, NeurIPS, and L4DC.
\end{IEEEbiography}

\end{document}
